\documentclass{article}

\PassOptionsToPackage{numbers,sort&compress}{natbib}
\usepackage[preprint]{neurips_2026}
\workshoptitle{E-Values: From Statistics to ML}

\usepackage[utf8]{inputenc}
\usepackage[T1]{fontenc}
\usepackage{amsmath,amssymb,amsfonts,amsthm,mathtools}
\usepackage{booktabs,array}
\usepackage{graphicx}
\usepackage{algorithm}
\usepackage{algorithmic}
\usepackage{microtype}
\usepackage{xcolor}
\usepackage{url}
\usepackage{hyperref}

\hypersetup{colorlinks=true,citecolor=blue!55!black,linkcolor=blue!55!black,urlcolor=blue!55!black}

\newcommand{\method}{e-ATS}
\newcommand{\anchor}{OTS}
\newcommand{\E}{\mathbb E}
\newcommand{\Pp}{\mathbb P}
\newcommand{\F}{\mathcal F}
\newcommand{\1}{\mathbf 1}
\newcommand{\kl}{\operatorname{kl}}
\newcommand{\tauth}{\tau_{\mathrm{auth}}}
\newcommand{\tdep}{\tau_{\mathrm{dep}}}
\newcommand{\alphaE}{\alpha_E}
\newcommand{\kappaM}{\kappa_M}
\newcommand{\epsM}{\varepsilon_M}
\newtheorem{theorem}{Theorem}
\newtheorem{proposition}[theorem]{Proposition}

\title{When May a Bandit Leave Its Anchor?\\
E-Process-Authorized Thompson Sampling under Non-stationarity}

\author{%
  Mayand Gulati\textsuperscript{1} \quad
  Kerong Wang\textsuperscript{1} \quad
  WeiChen Au\textsuperscript{2} \\
  \textsuperscript{1}UC Santa Barbara \qquad
  \textsuperscript{2}Purdue University \\
}

\begin{document}
\frenchspacing
\maketitle
\begin{abstract}
Stationarity rewards memory, but after change the same history can mislead.
We ask when forgetting should be permitted.
E-process-authorized Thompson
sampling (\method{}) gives each arm full-history and discounted Beta
states.
An anytime-valid e-process first authorizes the discounted
state, then a reversible relevance score controls its influence.
Before authorization, \method{} exactly follows optimistic Thompson sampling
(\anchor{}). Under a Beta--Bernoulli prior-predictive stationary model, \method{}'s
probability of ever departing from \anchor{} is at most the chosen $\alphaE$,
without fitted thresholds.
Relative to \method{}, removing authorization increased mean normalized
dynamic pseudo-regret by $38.4\%$ on the registered suite but reduced it
by $7.5\%$ on the literature-derived replay suite.
Therefore, evidence controls when
adaptation begins, not whether it always helps.
\end{abstract}

\section{When should a bandit stop trusting its history?}

In a non-stationary bandit, history can become misleading. Existing methods
discount or window past observations, adapt their forgetting rate, or restart
learning on a schedule
\citep{garivier2008ucb,raj2017taming,lu2017adaptive,besbes2014rexp3,
russac2019weighted,fiandri2024swts}. Detector-based methods instead wait for
evidence of change \citep{cao2018mucb,liu2017cusum,besson2022glr}. We ask
whether such evidence can also control the probability that a policy leaves a
trusted stationary anchor.

E-processes provide the statistical tool. Under a declared null model,
an e-process crosses $1/\delta$ with probability at most $\delta$, without
fixing a stopping time or horizon
\citep{ville1939,vovk2021evalues,grunwald2024safe,ramdas2023savi}. E-detectors
combine e-processes started at different times to monitor changes at unknown times
\citep{shin2023edetectors}. Mixtures over geometrically spaced candidate
change locations are also established \citep[Section~4.4]{ramdas2022exchangeability}.
We use the resulting e-process threshold crossing to authorize departure.
Our anchor is the Bernoulli form of Optimistic Bayesian
Sampling \citep{thompson1933,may2012optimistic}, which we call optimistic
Thompson sampling (\anchor{}). It uses each arm's full-history Beta posterior.
In \method{}, the discounted posterior
cannot affect an action before authorization, so the crossing bound also bounds
departure from the anchor.
Our contribution is a controller that translates statistical authorization into
prior-predictive control of departure from \anchor{}, sufficient evidence
conditions under adaptive sampling, and an empirical evaluation of
\method, including authorization's environment-dependent cost. When the anchor
is trusted under the declared model, limiting unnecessary departures is valuable.
Conservative bandits constrain reward relative to a baseline \citep{wu2016conservative}.
We constrain departure from the anchor’s trajectory.

\paragraph{Setup.}
There are $K$ arms and $T$ rounds. At round $t$, the policy selects $A_t$. If
this is arm $k$'s $n$th pull, its reward is
$Y_t=X_{k,n}\in\{0,1\}$, and $N_k(t)$ counts its pulls through round $t$. Our
stationary model is
\begin{equation}
 P_k\stackrel{\mathrm{iid}}{\sim}\mathrm{Beta}(1,1),\qquad
 X_{k,n}\mid P_k\stackrel{\mathrm{ind}}{\sim}\mathrm{Bernoulli}(P_k).
 \label{eq:null}
\end{equation}
Each arm draws one probability and keeps it forever. The guarantee averages
over these draws and does not hold separately for every fixed set of arm
probabilities.

\section{From evidence to action}
\label{sec:method}

We use this stationary model to decide when the policy may supplement its
full-history anchor with short-memory estimates.
Each arm maintains two Beta states initialized at $\mathrm{Beta}(1,1)$:
a full-history posterior that retains every reward from that arm, and a
discounted short-memory pseudo-posterior.
Each round, the short-memory counts above the prior are multiplied by
$\gamma_S=0.995$ before the played arm incorporates its reward,
giving an effective memory of about 200 rounds.
Evidence grants arm-specific authorization for the short-memory state
to affect action selection, relevance measures continuing disagreement,
and a decision weight determines its influence.

\paragraph{Has the arm produced evidence against stationarity?}
To decide whether arm $k$ can use its short-memory posterior, we construct an
arm-specific e-process that compares the declared stationary model with
models allowing for exactly one change at an unknown time.

The change time is unknown, so we consider several candidate change-points. Candidate $j$ places a possible change-point at arm-local pull $s_j$. The candidate models the rewards through $s_j - 1$ using one constant Bernoulli probability and models the rewards from pull $s_j$ onward using a second constant probability. The two probabilities may differ.
For a contiguous block of arm-local rewards
$X_{k,a:b}=(X_{k,a},\ldots,X_{k,b})$, let
$S_{a:b}=\sum_{\ell=a}^b X_{k,\ell}$ and
$F_{a:b}=b-a+1-S_{a:b}$ be the block's number of successes and failures.
Under a $\mathrm{Beta}(1,1)$ prior, the Beta--Bernoulli marginal likelihood of the block is
$m(X_{k,a:b})=\mathrm B(1+S_{a:b},1+F_{a:b})$.
The evidence for a split at change-point $s_j$, relative to a
stationary model over all $n$ pulls, is

\begin{equation}
 L_{k,n}^{(j)}=
 \frac{m(X_{k,1:s_j-1})\,m(X_{k,s_j:n})}{m(X_{k,1:n})}.
 \label{eq:ecomponent}
\end{equation}

Every pull of an arm could be a change-point, although this would require tracking $n-1$ candidates after $n$ pulls.
Instead, we use the geometric grid $s_1=2$ and $s_{j+1}=\lceil1.25\,s_j\rceil$.
If a change occurs at pull $s$, the most recent grid point $s_j\le s$ lies less than 20 percent
earlier.
So, every possible change-point has a nearby candidate while
only $O(\log n)$ candidates must be tracked rather than $O(n)$.

Since we do not know which candidate is closest to the actual change-point,
we assign candidate $j$ the fixed weight
$\pi_j=1/[j(j+1)]$ so
$\sum_j\pi_j=1$, and let $J(n)=\#\{j:s_j\le n\}$.
Unreached candidates have $L_{k,n}^{(j)}=1$, and their weights sum to
$1/(J(n)+1)$.
Therefore,

\begin{equation}
 E_{k,n}=\sum_{j=1}^{J(n)}\pi_j\,L_{k,n}^{(j)}
 \;+\;\frac{1}{J(n)+1}.
 \label{eq:monitor}
\end{equation}

The arm-specific evidence process $E_{k,n}$ starts at one and is a nonnegative martingale under \eqref{eq:null}. Arm $k$ becomes authorized when $E_{k,n}$ reaches $K/\alpha_E$, setting the authorization latch to $C^E_{k,t}=\1\{\max_{n\le N_k(t)}E_{k,n}\ge K/\alpha_E\}$.
The maximum keeps $C^E_{k,t}$ at one after its first crossing.

\paragraph{Do recent rewards still disagree with the long-history model?}

The authorization latch records whether the arm's evidence process has ever
crossed its threshold. Since the latch is permanent, it does not indicate
whether recent rewards still disagree with the long-history posterior's predictions. Therefore, we maintain a relevance score for the newly arriving rewards.

Before observing arm $k$'s $n$th reward, we record the long-history
posterior's prediction.
After observing the reward, we compute its residual
from the long-history posterior's prediction and update an exponentially weighted average with a memory of approximately $\sqrt T$ arm-local pulls.
We normalize the magnitude of this average by a corresponding scale, giving the relevance score $r_{k,n}\in[0,1]$.
Values of relevance near zero mean little sustained disagreement, while larger values indicate persistent recent disagreement
with the long-history model. Appendix~\ref{app:algorithm} gives the exact update and normalization.

\paragraph{How strongly should the short-memory posterior influence the choice?}

At round $t$, authorization and current relevance form the gate
\begin{equation}
g_{k,t}=C^E_{k,t-1} \cdot r_{k,N_k(t-1)},\qquad
R_{k,t}=\frac{n^S_{k,t}}{n^S_{k,t}+\rho},\qquad
w_{k,t}=g_{k,t}^{\nu}
+\zeta(1-g_{k,t}^{\nu})g_{k,t}R_{k,t},
\label{eq:weight}
\end{equation}
Here, $n^S_{k,t}$ is the amount of data retained by the short-memory
posterior, $\rho$ sets the retained-data scale, $\nu$ controls how strongly
the weight responds to the gate, and $\zeta$ controls how retained data further
increase the short-memory posterior's influence. Appendix~\ref{app:algorithm} gives their
fixed values.
For each arm, the decision weight $w_{k,t}$ is used to combine the two posterior draws and the two posterior means.
The arm's selection score is the larger of these two weighted combinations, and the policy selects the arm with the largest score.

Thus, authorization determines whether adaptation is permitted, relevance
gauges whether recent disagreement persists, and the decision weight
determines how strongly the policy adapts. Before any arm is authorized,
$g_{k,t}=w_{k,t}=0$ for every arm, so no short-memory draw is used and
\method{} exactly reproduces \anchor{} under the same long-posterior draws and
tie-breaking. Appendix~\ref{app:algorithm} gives the complete algorithm and
update order.

\section{When the policy can depart from its anchor}

We first show that under the declared stationary model, \method{} is unlikely
to receive permission to depart from \anchor{}.
We then give sufficient conditions for an arm's evidence to cross the
authorization threshold when its rewards change.

Let $\tauth$ be the first round where
$C^E_{k,t-1}=1$ for at least one arm, and let $\tdep$ be the first round where \method{} and \anchor{} select different arms.
We couple the policies by starting them from identical states and giving
them the same arm-specific reward sequences, policy-randomness stream,
and tie-breaking rule.

\begin{theorem}[Controlled departure]
\label{thm:main}
Pathwise, $\tdep\ge\tauth$, and before $\tauth$ the two policies have identical
actions, observed rewards, and long-posterior states. Under \eqref{eq:null},
\begin{equation}
 \Pp(\tdep\le T)\le\Pp(\tauth\le T)\le\alphaE.
 \label{eq:departure}
\end{equation}
\end{theorem}

\paragraph{Why does the guarantee hold?}
For each candidate change-point, $L_{k,n}^{(j)}$ is a likelihood-ratio
martingale under the prior-predictive stationary model.
Their fixed weighted average, $E_{k,n}$, is therefore also a nonnegative martingale starting at one.
Adaptive arm selection does not invalidate this property because an arm is
selected before its next reward is revealed.
Ville's inequality bounds each arm's crossing probability by $\alphaE/K$, and a union bound over the $K$ arms bounds the probability of any authorization by $\alphaE$.

Before authorization, every gate and decision weight is zero, so the coupled
policies have the same trajectory. A departure therefore requires an earlier
e-process crossing, turning the preceding statistic bound into
\eqref{eq:departure}. The resulting threshold $K/\alphaE$ requires no
horizon-specific fitting. \qed

The departure bound averages over the $\mathrm{Beta}(1,1)$ probabilities drawn
in \eqref{eq:null}. It does not hold separately for every fixed set of reward
probabilities, and it certifies permission to depart rather than the regret or
behavior afterward. Appendix~\ref{app:proof} gives the formal filtration and
coupling argument.

\paragraph{When is the evidence sufficient for authorization?}

Here, we ask when an actual change can produce enough evidence to authorize adaptation.
After $n$ pulls, candidate change-point $s_j$ divides the arm's rewards into segments of lengths $n_0=s_j-1$ and $n_1=n-s_j+1$.
Let $\widehat p_0$ and $\widehat p_1$ be their empirical means and $\widehat p$ the pooled mean.
Define
\[
I_{j,n}
=
n_0\kl(\widehat p_0\Vert\widehat p)
+
n_1\kl(\widehat p_1\Vert\widehat p),
\]
where $\kl$ is the Kullback--Leibler divergence between Bernoulli
distributions. $I_{j,n}$ is the gain in maximized log likelihood from
fitting separate reward probabilities before and after $s_j$, instead of one
probability to all $n$ rewards.

Appendix~\ref{app:evidence} proves
\[
\log E_{k,n}
\ge
I_{j,n}
-\log((n_0+1)(n_1+1))
-\log(j(j+1)).
\]
The two penalties come from averaging the segment probabilities and weighting
candidate $j$ by $1/[j(j+1)]$.
Theorem~\ref{thm:adaptive-power} in Appendix~\ref{app:evidence} gives a sufficient
condition for authorization across all candidates and completed observation
counts of one arm, even when earlier rewards determine when that arm is pulled.
Theorem~\ref{thm:finite-power} gives a sharper sufficient condition when the change occurs
at a fixed candidate and both segment lengths are chosen before seeing the rewards.
Because \method{} does not force exploration, neither theorem guarantees that
a changed arm will be sampled enough to satisfy the stated condition
or authorized within a fixed number of rounds.

\section{Experiments}
\label{sec:experiments}

We report mean normalized dynamic pseudo-regret
$\bar R_T=T^{-1}\sum_{t=1}^T(\max_k\mu_{k,t}-\mu_{A_t,t})$, where
$\mu_{k,t}$ is arm $k$'s mean reward. Lower is better. The registered suite
contains 16,380 non-stationary environments across nine mechanisms and 20
choices of $K$ and $T$, plus 1,820 stationary environments. The primary
comparator, the open ramp, removes the authorization latch and lets relevance
act immediately.
Before evaluation, we chose \method{}'s regret minus the open ramp's regret
as the primary contrast and required it to be below $-0.0002$.
The replay suite compares \method{} with 16 baselines across 4,500
matched runs on 606 mean trajectories from nine literature-derived
environment families.
We compare policies only within a suite.
Appendix~\ref{app:experiments} gives the full design and results.

\paragraph{In these stationary experiments, \method{} almost never leaves its anchor.}
Across 1,820 stationary environments, only one arm was authorized in total.
Mean regret was $0.01157$ for \method{} and \anchor{}, versus $0.02937$ for the
open ramp.

\begin{table}[H]
\centering
\small
\caption{Mean normalized dynamic pseudo-regret. Lower is better, and the
columns use different evaluation suites.}
\label{tab:primary-regret}
\begin{tabular}{@{}lrr@{}}
\toprule
Policy & Registered (16,380) & Replay (4,500) \\
\midrule
\anchor{}       & 0.0494 & 0.2157 \\
Open ramp     & 0.0540 & 0.1414 \\
\method{}     & 0.0390 & 0.1530 \\
Certified GLR & 0.0418 & 0.1750 \\
Practical GLR & 0.0392 & 0.1529 \\
\bottomrule
\end{tabular}
\end{table}

\paragraph{Authorization reduced regret on the registered suite and increased it on replay.}
Using the open ramp as the denominator, \method{}'s regret is $27.7\%$ lower on
the registered suite and $8.2\%$ higher on the replay suite. On the raw scale,
subtracting the open ramp's regret from \method{}'s gives $-0.01497$ with a 95\% interval of
$[-0.01533,-0.01461]$ on the registered suite and $+0.01154$ with an
interval of $[0.01089,0.01218]$ on the replay suite.
Appendix~\ref{app:experiments} examines the reversal.

\paragraph{The practical GLR is close, but its guarantee is different.}
We evaluate two alternatives that replace only e-process authorization
with generalized likelihood-ratio (GLR) rules \citep{besson2022glr}.
The certified GLR is valid for every fixed stationary Bernoulli mean,
beyond our prior-predictive null.
The practical GLR has no finite-sample guarantee.
On the registered and replay suites, regret differences relative
to \method{} are $+0.0028,+0.0221$ for the certified GLR
and $+0.0002,-0.00006$ for the practical GLR.
Tuning the GLR separately for 20 geometries
gives $0.0387$ regret versus \method{}’s $0.0390$ on the
registered suite, without an analytic guarantee.
Appendix~\ref{app:glr-comparator} gives details.

\section{What this establishes}

Under the declared model, the probability of departure is at most $\alpha_E$,
and expected normalized pseudo-regret differs from the anchor’s by at
most $\alpha_E$.
The policy forces no exploration and has mean rank 9.72 among 17 replay methods.
Future work should develop authorization that is restartable and valid for
fixed means, together with non-stationary regret guarantees.

\newpage
\bibliographystyle{plainnat}
\bibliography{references}

\newpage
\appendix
\raggedbottom
\section{Complete algorithm and update order}
\label{app:algorithm}

In Section~2 we explained how evidence is converted into an action.
This appendix provides the exact calculations and their order.
At the start of round $t$, every policy state depends only on observations completed through round
$t-1$.
The policy first scores the arms, then selects one arm and observes its reward, and finally updates its evidence and posterior states.

\paragraph{How is recent disagreement updated?}
Suppose arm $k$ is selected for the $n$th time and reward $X_{k,n}$ is observed.
Before updating either posterior with the new $X_{k,n}$, let $\widehat\mu^L_{k,n}$ be the mean of its long-history posterior.
We compare the reward with this prediction and update the following:
\begin{align}
 M_{k,n}
 &=
 \lambda_D M_{k,n-1}
 +(1-\lambda_D)(X_{k,n}-\widehat\mu^L_{k,n}),
 \nonumber\\[-2pt]
 Q_{k,n}
 &=
 \lambda_D^2Q_{k,n-1}+(1-\lambda_D)^2,
 \qquad
 s^M_{k,n}
 =
 \sqrt{\frac{Q_{k,n}}{2}
 \log\!\frac{K\pi^2n^2}{3\kappaM}} .
 \label{eq:mean}
\end{align}
Here,
\[
 \lambda_D
 =
 \min\{0.999,\max\{0,1-T^{-1/2}\}\}.
\]
Here, $M_{k,n}$ is an exponentially weighted average of the differences between observed rewards and the long-history posterior's predictions.
The measure $Q_{k,n}$ tracks the sum of the squared weights used in the average.
$Q_{k,n}$ therefore accounts for how much reward noise remains after averaging and is used to form the normalization scale $s^M_{k,n}$.
To clarify, $Q_{k,n}$ is not an empirical estimate of the residual variance.
The base value $\lambda_D=1-T^{-1/2}$ makes it so that approximately the most recent $\sqrt T$ observations of the arm contribute significantly to $M_{k,n}$.
The minimum and maximum in the definition of $\lambda_D$ restrict it to $[0,0.999]$ for extremely short or long horizons.
For the normalization scale $s^M_{k,n}$, the constant $\kappaM$ controls how the average prediction error is converted
into relevance.
A smaller value produces a larger $s^M_{k,n}$ and therefore less relevance for the same value of $M_{k,n}$.
A larger value has the opposite effect.
Empirically, we set $\kappaM=0.05$ because the relevance mechanism inherits its
prediction-error scale from an earlier version of the controller.
In that construction, $0.05$ was divided across the $K$ arms and all arm-local pull
counts when determining how large a prediction error should be considered
unusual.
We have retained this scale without retuning it on the primary
evaluation for \method{}.
The current policy uses it only to normalize $r_{k,n}$.

The relevance score is
\begin{equation}
 r_{k,n}
 =
 \1\{|M_{k,n}|>\epsM\}
 \min\{1,|M_{k,n}|/s^M_{k,n}\},
 \qquad
 \epsM=10^{-12}.
 \label{eq:ramp}
\end{equation}
Here, taking the absolute value of $M_{k,n}$ allows the score to respond to reward probabilities that move either up or down.
The cap keeps $r_{k,n}\in[0,1]$.
The constant $\epsM$ is used to guard against numerical degeneracy.

\paragraph{How is an arm scored before the next reward?}
At the beginning of round $t$, arm $k$ has long-history and short-memory posterior states
\[
 \mathrm{Beta}(\alpha^L_{k,t},\beta^L_{k,t})
 \quad\text{and}\quad
 \mathrm{Beta}(\alpha^S_{k,t},\beta^S_{k,t}).
\]
Their posterior means are denoted by
$\widehat\mu^L_{k,t}$ and $\widehat\mu^S_{k,t}$.
The amount of data retained above the short-memory posterior's initial $(1,1)$ state is
\[
 n^S_{k,t}
 =
 \max\{\alpha^S_{k,t}+\beta^S_{k,t}-2,0\}.
\]
Here, the maximum in $n^S_{k,t}$ protects against numerical roundoff.

Equation~\eqref{eq:weight} uses the current authorization latch, relevance, and
retained data to compute $g_{k,t}$, $R_{k,t}$, and $w_{k,t}$.
The term $g_{k,t}^{\nu}$ gives the short-memory posterior positive influence whenever the gate is positive.
The second term further increases the short-memory posterior's influence as the retained-data factor $R_{k,t}$ grows.
Thus, $R_{k,t}$ scales this additional contribution rather than acting as a separate gate.
When $g_{k,t}=0$, the decision weight $w_{k,t}$ is zero.
On the other hand, when $g_{k,t}=1$, the decision weight $w_{k,t}$ is one.

The policy always draws
\[
 Z^L_{k,t}
 \sim
 \mathrm{Beta}(\alpha^L_{k,t},\beta^L_{k,t}).
\]
It draws
$Z^S_{k,t}\sim
\mathrm{Beta}(\alpha^S_{k,t},\beta^S_{k,t})$
only when $w_{k,t}>0$. The superscripts $S$ and $L$ denote the short-memory and long-history posteriors, respectively.
The arm's score is
\begin{equation}
 U_{k,t}
 =
 \begin{cases}
 \max\{Z^L_{k,t},\widehat\mu^L_{k,t}\},
 & w_{k,t}=0,\\[3pt]
 \max\!\left\{
 (1-w_{k,t})Z^L_{k,t}+w_{k,t}Z^S_{k,t},
 (1-w_{k,t})\widehat\mu^L_{k,t}
 +w_{k,t}\widehat\mu^S_{k,t}
 \right\},
 & w_{k,t}>0.
 \end{cases}
 \label{eq:score}
\end{equation}
The maximum in $U_{k,t}$ preserves the optimistic clipping used by \anchor{}.
For any arm with $w_{k,t}=0$, the score reduces to $\max\{Z^L_{k,t},\widehat\mu^L_{k,t}\}$, which is exactly the score used by
\anchor{}.
Before any arm is authorized, this holds for every arm, so \method{} and \anchor{} select the same action when
given the same long-posterior draws.
This also explicitly follows from the coupling in Theorem~\ref{thm:main}.

\paragraph{How does one round proceed?}
Algorithm~\ref{alg:eats} gives the complete sequence of decisions and updates.
At the beginning of each round, the policy computes every arm's score using only rewards observed in earlier rounds.
It then selects an arm and observes its reward.
For the selected arm, the policy computes the prediction error using the long-history posterior mean recorded before that reward was observed.
The policy then updates the relevance score and e-process, discounts the short-memory posteriors, and finally adds the new reward to the selected arm's long- and short-memory posteriors.

\begin{algorithm}[H]
\caption{E-process-authorized Thompson sampling (\method{})}
\label{alg:eats}
\begin{algorithmic}[1]
\REQUIRE Arms $K$, horizon $T$, $\alphaE=0.05$, $\kappaM=0.05$,
$\gamma_S=0.995$, $\nu=0.3$, and $\zeta=0.8$
\STATE Set
$\lambda_D=\min\{0.999,\max\{0,1-T^{-1/2}\}\}$
\STATE Set
$\rho=\max\{4,16\sqrt{(T/5000)(4/K)}\}$
\STATE Set $s_1=2$, $s_{j+1}=\lceil1.25s_j\rceil$, and
$\pi_j=1/[j(j+1)]$
\FOR{$k=1,\ldots,K$}
  \STATE Set
  $(\alpha^L_k,\beta^L_k)=(\alpha^S_k,\beta^S_k)=(1,1)$
  \STATE Set
  $E_{k,0}=1$, $C^E_k=0$, and $M_k=Q_k=r_k=N_k=0$
\ENDFOR
\FOR{$t=1,\ldots,T$}
  \FOR{$k=1,\ldots,K$}
    \STATE Compute $\widehat\mu^L_{k,t}$,
    $\widehat\mu^S_{k,t}$, and $n^S_{k,t}$
    \STATE Compute $g_{k,t}$, $R_{k,t}$, and $w_{k,t}$
    from Equation~\eqref{eq:weight}
    \STATE Draw $Z^L_{k,t}$
    \IF{$w_{k,t}>0$}
      \STATE Draw $Z^S_{k,t}$
    \ENDIF
    \STATE Compute $U_{k,t}$ from Equation~\eqref{eq:score}
  \ENDFOR
  \STATE Select $A_t=\arg\max_k U_{k,t}$ using the fixed tie rule
  \STATE Observe reward $Y_t$
  \STATE Set $k=A_t$, $n=N_k+1$, and $X_{k,n}=Y_t$
  \STATE Record
  $\widehat\mu^L_{k,n}\leftarrow\widehat\mu^L_{k,t}$
  \STATE Update $M_{k,n}$, $Q_{k,n}$, and $r_{k,n}$
  from Equations~\eqref{eq:mean}--\eqref{eq:ramp}
  \STATE Update $E_{k,n}$ from
  Equations~\eqref{eq:ecomponent}--\eqref{eq:monitor}
  \IF{$E_{k,n}\ge K/\alphaE$}
    \STATE Set $C^E_k\leftarrow1$
  \ENDIF
  \FOR{$\ell=1,\ldots,K$}
    \STATE
    $\alpha^S_\ell\leftarrow
    1+\gamma_S(\alpha^S_\ell-1)$
    \STATE
    $\beta^S_\ell\leftarrow
    1+\gamma_S(\beta^S_\ell-1)$
  \ENDFOR
  \STATE
  $\alpha^S_k\leftarrow\alpha^S_k+Y_t$
  and
  $\beta^S_k\leftarrow\beta^S_k+1-Y_t$
  \STATE
  $\alpha^L_k\leftarrow\alpha^L_k+Y_t$
  and
  $\beta^L_k\leftarrow\beta^L_k+1-Y_t$
  \STATE Set $N_k\leftarrow n$
\ENDFOR
\end{algorithmic}
\end{algorithm}

Since the candidate change-points lie on a geometric grid, each arm tracks
only $O(\log T)$ candidates.
This gives $O(K\log T)$ total stored state and $O(K+\log(1+N_{A_t}(t)))$ computation per round, where $N_{A_t}(t)$ is the
number of times the selected arm $A_t$ has been pulled through round $t$.

\paragraph{Which parts of the algorithm are required for the guarantee?}
The controlled-departure guarantee depends on the prior-predictive likelihood
ratios, fixed nonnegative candidate weights, the weight
reserved for candidates not yet reached, and the authorization threshold
$K/\alphaE$.
The remaining constants affect how the short-memory posterior is used after authorization.
They can change the policy's empirical behavior, but not the e-process crossing bound.
The fixed values and the formulas that depend on $K$ and $T$ were specified before the primary evaluation in Section~\ref{sec:experiments}.
The same prespecified configuration was used across all registered environments and geometries, without environment-specific retuning after outcomes were observed.
Appendix~\ref{app:experiments} describes the freeze and the separation between development and evaluation data.
The primary evaluation therefore assesses one fixed policy specification rather than selecting the best configuration from the evaluation results.
These constants are not claimed to be optimal.

The experiments use the known horizon $T$ to set $\lambda_D$ and $\rho$.
Experiments with an unknown horizon were not evaluated.
Using a nominal horizon to set $\lambda_D$ and $\rho$ does not change the controlled-departure guarantee because neither
constant enters the e-process or its authorization threshold.
The nominal horizon may change behavior after authorization, and that alternative was not evaluated here.
Only $\alphaE$ is an error probability.

Algorithm~\ref{alg:eats} is the fixed specification evaluated in
Section~\ref{sec:experiments}.
Before authorization, \method{} follows \anchor{} exactly, while afterward the permanent latch permits adaptation and
the relevance score controls how strongly the short-memory posterior affects arm selection.

\section{Proof of controlled departure}
\label{app:proof}

Theorem~\ref{thm:main} combines a statistical result with an algorithmic one.
We first show that each arm's evidence remains a martingale when the policy chooses arms.
We then bound the probability that any arm becomes authorized.
Finally, we show that \method{}'s trajectory cannot depart from \anchor{}'s trajectory before such an authorization.

\paragraph{Why is the arm-specific evidence a martingale?}
For arm $k$, let
\[
\mathcal H_{k,n}
=
\sigma(X_{k,1},\ldots,X_{k,n})
\]
contain its first $n$ observed rewards.
Under the prior-predictive stationary model in Equation~\eqref{eq:null}, $m(X_{k,1:n})$ is the probability assigned to this reward sequence by one unknown Bernoulli probability.
As described in Section~\ref{sec:method}, candidate change-point $s_j$ marks a possible boundary between two periods in arm $k$'s reward sequence.
Rewards before $s_j$ use one constant Bernoulli probability, while rewards from $s_j$ onward use a second probability that may differ.
Considering several candidate change-points allows the evidence process to look for departures from stationarity without knowing when a change occurred.
The numerator of Equation~\eqref{eq:ecomponent} assigns independent $\mathrm{Beta}(1,1)$ priors to the probabilities before and after
$s_j$.

For arm $k$ after its first $n$ rewards have been observed, $L_{k,n}^{(j)}$ is the likelihood ratio comparing change-point candidate $j$, which places a change at arm-local pull $s_j$, with the stationary model over those same $n$ rewards.
Before pull $s_j$, the proposed change has not occurred, so the two models assign the same probability and $L_{k,n}^{(j)}=1$.
Under the stationary model,
\[
\E\!\left[
L_{k,n}^{(j)}
\mid
\mathcal H_{k,n-1}
\right]
=
L_{k,n-1}^{(j)}.
\]
Therefore, each candidate's evidence is a nonnegative martingale starting at
one.

Candidates whose proposed change-points have not yet been reached remain at
one.
Equation~\eqref{eq:monitor} can consequently be written as
\[
E_{k,n}
=
\sum_{j\geq1}\pi_jL_{k,n}^{(j)}.
\]
Here, $\pi_j=1/[j(j+1)]$ is the weight assigned to candidate $j$ and
determines how much candidate change-point $s_j$ contributes to $E_{k,n}$.
The sequence $(\pi_j)_{j\geq1}$ is nonnegative and sums to one.
Since every term is nonnegative, conditional monotone convergence gives
\[
\E\!\left[E_{k,n}\mid\mathcal H_{k,n-1}\right]
=
\sum_{j\geq1}\pi_j
\E\!\left[
L_{k,n}^{(j)}
\mid
\mathcal H_{k,n-1}
\right]
=
E_{k,n-1}.
\]
Hence, $E_{k,n}$ is also a nonnegative martingale starting at one.
The final term $1/(J(n)+1)$ in Equation~\eqref{eq:monitor} is exactly the combined
weight of the candidate change-points not yet reached.

\paragraph{Why does adaptive arm selection preserve this property?}
Fix an arm $k$.
Let $A_t$ denote the arm selected by the policy in calendar round $t$.
So, the evidence for arm $k$ is updated in that round only when $A_t=k$.

Let $U_t$ denote all posterior samples drawn while the policy scores the arms
in round $t$.
This includes one long-history posterior sample for every arm and a short-memory posterior sample for each arm with $w_{k,t}>0$.
Conditional on $\F_{t-1}$, these samples are generated using fresh policy
randomness and reveal no information about any arm's next unrevealed reward.

Define
\[
\F_{t-1}
=
\sigma(U_1,A_1,Y_1,\ldots,U_{t-1},A_{t-1},Y_{t-1}),
\]
where $\F_0$ represents the empty history before round one.
Therefore, $\F_{t-1}$ contains everything revealed through the end of round $t-1$.
After the samples in $U_t$ are drawn and arm $A_t$ is selected, but before its
reward is observed, the available information is
\[
\mathcal G_t
=
\F_{t-1}\vee\sigma(U_t,A_t).
\]
The selected arm then produces reward
\[
Y_t=X_{A_t,N_{A_t}(t)},
\]
where $N_{A_t}(t)$ is the number of times that arm has been pulled through
round $t$.
After this reward is observed, the available information is
\[
\F_t
=
\mathcal G_t\vee\sigma(Y_t).
\]
Therefore, $A_t$ is selected using only information available before $Y_t$ is
revealed.

For each arm, define its calendar-time evidence by
\[
\widetilde E_{k,t}=E_{k,N_k(t)}.
\]
Its update can therefore be written as
\[
\widetilde E_{k,t}
=
\widetilde E_{k,t-1}
+
\1\{A_t=k\}
\left(
E_{k,N_k(t-1)+1}-E_{k,N_k(t-1)}
\right).
\]
If arm $k$ is not selected, its evidence does not change.
On the other hand, if it is selected, the next reward has the prior-predictive distribution determined by that arm's observed history.
The histories of the other arms and the independent policy randomness reveal no additional information about this
unseen reward.
Therefore,
\[
\E\!\left[
\widetilde E_{k,t}
\mid
\mathcal G_t
\right]
=
\widetilde E_{k,t-1}.
\]
Taking the conditional expectation again given $\F_{t-1}$ shows that
$\widetilde E_{k,t}$ remains a nonnegative martingale in calendar time.

Ville's inequality now gives, for each arm,
\[
\Pp\!\left(
\max_{t\leq T}\widetilde E_{k,t}
\geq
K/\alphaE
\right)
\leq
\frac{\alphaE}{K}.
\]
A union bound over the $K$ arms then gives
\[
\Pp\!\left(
\max_{k\leq K}\max_{t\leq T}\widetilde E_{k,t}
\geq
K/\alphaE
\right)
\leq
\sum_{k=1}^K\frac{\alphaE}{K}
=
\alphaE.
\]

\paragraph{Why can departure occur only after authorization?}
This final step is pathwise and does not require the stationary model.
We couple \method{} and \anchor{} from identical initial states using the same
arm-specific reward sequences, policy-randomness stream, and tie-breaking rule.

Before $\tauth$, every authorization latch is zero.
Equation~\eqref{eq:weight} then gives $g_{k,t}=w_{k,t}=0$ for every arm.
By Equation~\eqref{eq:score}, \method{} draws no short-memory samples and computes
exactly the same scores as \anchor{}.
The two policies therefore select the same arm, observe the same reward, and make the same long-posterior update.
Repeating this argument across rounds gives identical actions, rewards, and long-posterior states through every round before $\tauth$.

At round $\tauth$, the authorization latch is active, yet the policies need not differ in their trajectories.
It may be that the relevance score is still zero, or the modified scores have the same maximizing arm.
Hence, pathwise,
\[
\tdep\geq\tauth.
\]

An authorization latch can become active only after an arm's evidence has
crossed $K/\alphaE$ on a completed update.
Taking all this together,
\[
\{\tdep\leq T\}
\subseteq
\{\tauth\leq T\}
\subseteq
\left\{
\max_{k\leq K}\max_{t\leq T}\widetilde E_{k,t}
\geq K/\alphaE
\right\}.
\]
Combining this pathwise inclusion with the preceding probability bound gives
\[
\Pp(\tdep\leq T)
\leq
\Pp(\tauth\leq T)
\leq
\alphaE,
\]
which proves Theorem~\ref{thm:main}. \qed

\paragraph{What does the departure bound transfer?}
Let $H_T^e$ and $H_T^a$ denote the coupled histories of \method{} and
\anchor{}, including their actions, observed rewards, and shared latent means.
Let $L(H_T)\in[0,1]$ be any measurable function of a policy trajectory,
such as normalized dynamic pseudo-regret.
Applying the same function to the coupled \method{} and \anchor{} trajectories
gives
\[
\left|\E_0 L(H_T^e)-\E_0 L(H_T^a)\right|
\leq\Pp_0(\tdep\leq T)\leq\alphaE.
\]
This includes normalized dynamic pseudo-regret under the declared model.
To clarify, this bound may be loose and provides no guarantee after
conditioning on departure or outside the declared model.

\section{Evidence strength and authorization}
\label{app:evidence}

For each arm $k\in\{1,\ldots,K\}$, Section~\ref{sec:method} defines the
candidate-specific evidence $L_{k,n}^{(j)}$ and the arm-specific e-process
$E_{k,n}$ after $n$ pulls.
This appendix answers three questions.
How much evidence does one candidate provide?
How many observations can be sufficient for authorization after a change?
Does the authorization latch materially affect the controller?
All observation counts in this section are local to one arm.

\paragraph{How much evidence does one candidate provide?}
Recall that $j$ indexes the candidate change-points.
Candidate $j$ models the rewards through arm-local pull $s_j-1$ using one constant Bernoulli success
probability and the rewards from pull $s_j$ onward using a second probability that may be different from the first.
So, $s_j$ places the proposed change between pulls $s_j-1$ and $s_j$.
For a candidate with $s_j\leq n$, let
\[
 n_0=s_j-1,
 \qquad
 n_1=n-n_0
\]
be the segment lengths, and let $S_0$ and $S_1$ be their success counts.
Their empirical means and pooled mean are
\[
 \widehat p_0=\frac{S_0}{n_0},
 \qquad
 \widehat p_1=\frac{S_1}{n_1},
 \qquad
 \widehat p=\frac{S_0+S_1}{n}.
\]
Define
\[
 I_{j,n}
 =
 n_0\kl(\widehat p_0\Vert\widehat p)
 +
 n_1\kl(\widehat p_1\Vert\widehat p),
\]
where
\[
 \kl(p\Vert q)
 =
 p\log\frac pq
 +(1-p)\log\frac{1-p}{1-q}
\]
is the Kullback--Leibler divergence between Bernoulli distributions, with the
usual boundary conventions.
The quantity $I_{j,n}$ is the gain in maximized log likelihood from fitting one Bernoulli success probability to rewards
$X_{k,1:s_j-1}$ and another probability to rewards $X_{k,s_j:n}$.
This is instead of fitting one common probability to all $n$ observed rewards.

Before stating the result, recall that
\[
\pi_j=\frac{1}{j(j+1)}
\]
is the weight assigned to candidate change-point $j$.
The user-chosen level $\alphaE$ determines the authorization threshold $K/\alphaE$.

For the capacity statement, fix candidate $j$ and let
\[
\ell=s_j-1
\]
be the number of rewards observed before its proposed change-point.
Suppose these $\ell$ rewards contain $S$ successes and $\ell-S$ failures.
As the arm continues to be sampled, suppose the empirical success rate of its
rewards from pull $s_j$ onward converges to $\theta_1\in(0,1)$.
Thus, $\theta_1$ is the limiting success rate after the candidate change-point.
We will continue to use $\mathrm{B}$ for the beta function.

\begin{theorem}[Evidence lower bound and candidate capacity]
\label{thm:evidence-budget}
For every candidate $s_j\leq n$,
\begin{equation}
 \log E_{k,n}
 \geq
 I_{j,n}
 -\log((n_0+1)(n_1+1))
 -\log(j(j+1)).
 \label{eq:information-budget}
\end{equation}
For the fixed candidate described above, if the arm continues to be sampled,
then
\begin{equation}
 L_{k,\ell+m}^{(j)}
 \longrightarrow
 \frac{\mathrm{B}(S+1,\ell-S+1)}
 {\theta_1^S(1-\theta_1)^{\ell-S}}
 \qquad\text{as }m\to\infty.
 \label{eq:component-capacity}
\end{equation}
We call the limiting weighted evidence of candidate $j$, scaled by
$\alphaE$, its capacity:
\[
 \operatorname{Cap}_j
 =
 \alphaE\pi_j
 \frac{\mathrm{B}(S+1,\ell-S+1)}
 {\theta_1^S(1-\theta_1)^{\ell-S}}.
\]
If $\operatorname{Cap}_j>K$, candidate $j$ eventually authorizes the arm.
\end{theorem}

In Equation~\eqref{eq:component-capacity}, the powers $S$ and $\ell-S$ are the
numbers of successes and failures observed before the proposed change.
As a consequence,
$\theta_1^S(1-\theta_1)^{\ell-S}$ is the likelihood that a Bernoulli model with success
probability $\theta_1$ assigns to the observed pre-change sequence with $S$ successes and $\ell-S$ failures.
The quantity $\operatorname{Cap}_j$ includes the candidate weight
$\pi_j$ and the authorization level $\alphaE$.
So, the condition $\operatorname{Cap}_j>K$ is equivalent to the limiting contribution of
candidate $j$ exceeding the authorization threshold $K/\alphaE$.

\paragraph{Why does the bound hold?}
For a consecutive block of $h\geq1$ binary rewards from one arm, let $s$ be
the number of successes and let $\widehat p=s/h$ be the empirical success
rate.
The likelihood of the observed block is maximized at $\widehat p$, giving
\[
\widehat p^s(1-\widehat p)^{h-s}.
\]
Under a $\mathrm{Beta}(1,1)$ prior, averaging the same likelihood over the
unknown success probability gives
\[
\mathrm{B}(s+1,h-s+1)
=
\int_0^1 p^s(1-p)^{h-s}\,dp.
\]
We call the difference between these two log likelihoods the redundancy:
\begin{equation}
 \operatorname{red}(h,s)
 =
 \log
 \frac{\widehat p^s(1-\widehat p)^{h-s}}
 {\mathrm{B}(s+1,h-s+1)}.
 \label{eq:redundancy}
\end{equation}
Redundancy is the cost of averaging over the unknown success probability rather than
fitting that probability after observing the block.
The endpoint cases $s=0$ and $s=h$ are interpreted by continuity, using $0^0=1$.
Since
\[
 \exp\{\operatorname{red}(h,s)\}
 =
 (h+1)
 \Pr_{\widehat p}\{\operatorname{Bin}(h,\widehat p)=s\},
\]
we get
\begin{equation}
 0
 \le
 \operatorname{red}(h,s)
 \le
 \log(h+1).
 \label{eq:redundancy-bound}
\end{equation}

Writing $S=S_0+S_1$, direct substitution into
Equation~\eqref{eq:ecomponent} gives
\begin{equation}
 \log L_{k,n}^{(j)}
 =
 I_{j,n}
 -\operatorname{red}(n_0,S_0)
 -\operatorname{red}(n_1,S_1)
 +\operatorname{red}(n,S).
 \label{eq:exact-information}
\end{equation}
Since $E_{k,n}\ge\pi_jL_{k,n}^{(j)}$ and
$\pi_j=1/[j(j+1)]$, Equation~\eqref{eq:redundancy-bound} gives
Equation~\eqref{eq:information-budget}.

We next prove the limiting value in
Equation~\eqref{eq:component-capacity}.
After the fixed prefix of $\ell$ rewards, let $R_m$ be the number of successes among the next $m$ rewards.
Then
\[
 L_{k,\ell+m}^{(j)}
 =
 \mathrm B(S+1,\ell-S+1)
 \frac{
 \mathrm B(R_m+1,m-R_m+1)
 }{
 \mathrm B(S+R_m+1,\ell-S+m-R_m+1)
 }.
\]
The beta function satisfies
\[
\mathrm{B}(a,b)
=
\frac{\Gamma(a)\Gamma(b)}{\Gamma(a+b)},
\]
where $\Gamma$ is the gamma function and $\Gamma(q+1)=q!$ for every
nonnegative integer $q$.
If $R_m/m\to\theta_1$, substituting this identity
into the preceding expression and applying the asymptotic relation
\[
\frac{\Gamma(x+a)}{\Gamma(x+b)}
\sim
x^{a-b}
\qquad\text{as }x\to\infty
\]
gives Equation~\eqref{eq:component-capacity}.
When
$\operatorname{Cap}_j>K$, the limiting value of
$\pi_jL_{k,\ell+m}^{(j)}$ exceeds $K/\alphaE$, which forces authorization. \qed

So, a fixed candidate has a finite evidence capacity determined by the
rewards before its proposed change-point and the limiting success rate
afterward.
The condition $\operatorname{Cap}_j>K$ guarantees eventual
authorization when the arm continues to be sampled.
The converse does not hold because several candidates may combine to cross the threshold, and a
candidate may cross temporarily before settling below it.

\paragraph{How does capacity change with more observations before the proposed split?}
Consider a sequence of experiments indexed by $j$, where experiment $j$'s
change occurs after $\ell_j=s_j-1$ observations.
Suppose its prefix contains $S_j$ successes, $S_j/\ell_j\to\theta_0$, and its post-change empirical mean
converges to $\theta_1$, where $\theta_0,\theta_1\in(0,1)$.
From Equation~\eqref{eq:redundancy},
\[
\log\operatorname{Cap}_j
=
\ell_j\kl(S_j/\ell_j\Vert\theta_1)
-\operatorname{red}(\ell_j,S_j)
+\log\alphaE-\log(j(j+1)).
\]
Since $0\leq\operatorname{red}(\ell_j,S_j)\leq\log(\ell_j+1)$ and the grid
grows geometrically, $\log(j(j+1))/\ell_j\to0$.
Therefore,
\[
\frac{1}{\ell_j}\log\operatorname{Cap}_j
\longrightarrow
\kl(\theta_0\Vert\theta_1).
\]
When $\theta_0\neq\theta_1$, capacity therefore grows exponentially with
prefix length at rate $\kl(\theta_0\Vert\theta_1)$.
This characterizes the candidate's eventual evidence strength after continued sampling.
Before convergence, its evidence may temporarily exceed this limit.

\paragraph{Does the condition still hold when sampling depends on earlier rewards?}

The policy may use earlier rewards to decide when to pull an arm.
We therefore need a condition that remains valid when the observation
times depend on those earlier rewards.
The idea is the same as in Section~3 where a difference
between the two segments of one arm's observed reward sequence must be large enough to overcome sampling
noise and the evidence penalties.

Fix an arm $k$ and an allowed failure probability $\xi\in(0,1)$
before observing the rewards.
As in Appendix~\ref{app:proof}, $\mathcal G_t$ records all past arm
selections, observed rewards and policy randomization, together with
the current randomization and selected arm $A_t$, before reward $Y_t$
is observed.
Once reward $Y_t$ is observed, $\F_t$ contains all information in $\mathcal G_t$
as well as reward $Y_t$.
On rounds when $A_t=k$, let
\[
\mu_{k,t}=\Pp(Y_t=1\mid\mathcal G_t).
\]
So, $\mu_{k,t}$ is the reward's conditional success probability
given the information available before it is observed.
The policy does not need to know its value.
For arm $k$'s $i$th observed reward, let $\mu_i$ denote this probability
immediately before that reward is observed.

Let $n$ be the number of rewards observed from arm $k$. Candidate $j$ divides its observations
into the same two segments used above, with
$n_0=s_j-1$ observations before the candidate and
$n_1=n-s_j+1$ observations from the candidate onward.
To clarify, these are counts of observed rewards, not calendar rounds.
Define the difference between the two segments' average conditional
success probabilities by
\[
D_{j,n}
=
\left|
\frac{1}{n_0}\sum_{i=1}^{n_0}\mu_i
-
\frac{1}{n_1}\sum_{i=n_0+1}^{n}\mu_i
\right|.
\]
This is a difference in the reward means under the assumed model,
rather than the empirical difference calculated from the rewards.

To allow for sampling noise in both segments, let the margin be
\[
a_{j,n}
=
\sqrt{\frac{\log(4j(j+1)/\xi)}{2n_0}}
+
\sqrt{\frac{\log(4j(j+1)n_1(n_1+1)/\xi)}{2n_1}}.
\]
The first term bounds the error in the first segment's empirical mean, while
the second term does the same for the second segment.
The factor $j(j+1)$ accounts for checking every candidate location.
The additional factor $n_1(n_1+1)$ accounts for checking each
candidate again as more rewards are collected.
Therefore, one event of probability at least $1-\xi$ controls both errors
for every completed split.
On the event where both sampling errors satisfy their bounds,
the absolute difference between the two empirical reward means is at
least $\max\{D_{j,n}-a_{j,n},0\}$.

\begin{theorem}[Sufficient evidence under adaptive sampling]
\label{thm:adaptive-power}
For the fixed arm $k$, let $N_k(t)$ count the rewards observed
from that arm through round $t$.
Under the conditional reward law above, with probability at least
$1-\xi$, simultaneously for every candidate $j$ and every observation
count $s_j\leq n\leq N_k(T)$, the condition
\begin{equation}
\begin{split}
2\frac{n_0n_1}{n_0+n_1}(D_{j,n}-a_{j,n})_+^2
\geq{}&
\log(K/\alphaE)+\log(j(j+1))
\\
&+\log(n_0+1)+\log(n_1+1)
\end{split}
\label{eq:adaptive-power}
\end{equation}
implies $E_{k,n}\geq K/\alphaE$.
Consequently, arm $k$ is authorized by the end of the round in which
its $n$th reward is observed.
The probability includes both the rewards and the policy's randomization.
\end{theorem}

\begin{proof}
First, we bound sampling noise for a block of observations even when
the policy chooses when to collect them.
Fix integers $u\geq0$ and $h\geq1$.
The block consists of arm $k$'s next $h$ rewards after its first $u$
observations, namely observations $u+1,\ldots,u+h$.
Let
\[
I_t=\1\{A_t=k,\ u\leq N_k(t-1)<u+h\}
\]
indicate whether round $t$ contributes a reward to this block.
This is determined before $Y_t$ is observed.
Write
\[
Z_t=\sum_{r=1}^{t}I_r(Y_r-\mu_{k,r}),
\qquad
V_t=\sum_{r=1}^{t}I_r
\]
for the block's accumulated deviation from its conditional means
and its number of observations through round $t$.
For these sums, set $\mu_{k,r}=0$ when $A_r\ne k$; those rounds
have $I_r=0$ and contribute nothing.
Once the block is complete, $V_t=h$ and both sums stop changing.

For any fixed $\lambda\in\mathbb R$, conditional Hoeffding's lemma
gives
\[
\E\!\left[
\exp\{\lambda I_t(Y_t-\mu_{k,t})-\lambda^2I_t/8\}
\mid\mathcal G_t
\right]\leq1.
\]
So, a selected reward lies in $[0,1]$ and has conditional mean
$\mu_{k,t}$.
When $I_t=0$, the exponential equals one.
It follows, by taking conditional expectation again given $\F_{t-1}$,
that
\[
\mathcal M_t(\lambda)
=\exp\{\lambda Z_t-\lambda^2V_t/8\}
\]
is a nonnegative supermartingale starting at one.
This is the standard argument for bounded observations
\citep{howard2020timeuniform}.

If the block completes with $Z_t\geq ha$ for some fixed $a>0$,
then $\mathcal M_t(4a)\geq\exp(2ha^2)$.
Ville's inequality bounds the probability of this event by
$\exp(-2ha^2)$.
Using $\lambda=-4a$ gives the same bound for $Z_t\leq-ha$.
Therefore,
\begin{equation}
\Pp\!\left(\exists t\leq T:\ V_t=h,\ |Z_t|\geq ha\right)
\leq2\exp(-2ha^2).
\label{eq:adaptive-block-bound}
\end{equation}
This bounds the joint event that the block completes and its average
deviation is large.
If the policy never finishes the block, the event does not occur.

Apply this bound to each candidate's first segment with $u=0$ and
$h=n_0$, using the first term of $a_{j,n}$ as $a$.
The failure probability is at most $\xi/[2j(j+1)]$.
For its second segment, take $u=n_0$ and $h=n_1$, using the second
term as $a$.
The failure probability is at most
$\xi/[2j(j+1)n_1(n_1+1)]$.
For a fixed candidate, the first segment always contains the same observations.
So, its failure probability is counted only once in the union bound,
even as additional observations extend the second segment.
Since $\sum_{j\geq1}1/[j(j+1)]=1$, and the same identity holds
when summing over $n_1\geq1$, a union bound gives total failure
probability at most
\[
\sum_{j\geq1}\frac{\xi}{2j(j+1)}
+
\sum_{j\geq1}\sum_{n_1\geq1}
\frac{\xi}{2j(j+1)n_1(n_1+1)}
=\xi.
\]

With probability at least $1-\xi$, both empirical means satisfy their
stated error bounds at every completed split.
We now show that whenever all these bounds hold, the proposed condition
implies authorization.
Writing these empirical means as $\widehat p_0$ and $\widehat p_1$,
the triangle inequality gives
\[
|\widehat p_0-\widehat p_1|
\geq(D_{j,n}-a_{j,n})_+.
\]
Let $\widehat p=(n_0\widehat p_0+n_1\widehat p_1)/(n_0+n_1)$
be their pooled mean.
Pinsker's inequality then gives us
\begin{align*}
I_{j,n}
&=n_0\kl(\widehat p_0\Vert\widehat p)
+n_1\kl(\widehat p_1\Vert\widehat p)
\\
&\geq2n_0(\widehat p_0-\widehat p)^2
+2n_1(\widehat p_1-\widehat p)^2
\\
&=2\frac{n_0n_1}{n_0+n_1}(\widehat p_0-\widehat p_1)^2
\\
&\geq2\frac{n_0n_1}{n_0+n_1}(D_{j,n}-a_{j,n})_+^2.
\end{align*}
Finally, Equation~\eqref{eq:information-budget} yields
\[
\log E_{k,n}
\geq I_{j,n}-\log(n_0+1)-\log(n_1+1)-\log(j(j+1)).
\]
If Equation~\eqref{eq:adaptive-power} holds, the preceding
inequalities give
\[
\log E_{k,n}\geq\log(K/\alphaE),
\qquad\text{and therefore}\qquad
E_{k,n}\geq K/\alphaE.
\]
By the definition of the authorization latch, arm $k$ is therefore
authorized by the end of the round in which its $n$th reward
is observed.

The union bound shows that, with probability at least $1-\xi$,
all the sampling-error bounds used above hold together for every
candidate and every observation count completed by round $T$.
Whenever these bounds hold, Equation~\eqref{eq:adaptive-power}
implies authorization for every candidate and completed observation
count that satisfies it.
This proves the simultaneous guarantee.
\end{proof}

\paragraph{What does this condition guarantee?}

The theorem gives a condition under which one arm's observed rewards
are enough for authorization.
For the theorem, only the means in $D_{j,n}$ are used for the analysis.
The policy computes its evidence directly from the observed rewards.
Since \method{} may not pull the arm enough, the theorem does not
guarantee authorization or bound the number of rounds needed for authorization.

With probability at least $1-\xi$, every candidate split and completed observation
count of the fixed arm that satisfies the sufficient condition also satisfies
$E_{k,n}\ge K/\alpha_E$, implying authorization.
To give the same guarantee across all arms, we would divide $\xi$
among them.
To clarify, the separate parameter $\alphaE$ bounds the probability of leaving
the anchor under the stationary model in Theorem~\ref{thm:main}.

Theorem~\ref{thm:adaptive-power} is stated in arm-local observation time.
We next connect it to a change that occurs at a calendar round and falls
between candidate locations on the geometric grid.

Suppose arm $k$'s conditional mean changes from fixed $\theta_0$ to fixed
$\theta_1$ after calendar round $\tau$.
By then, the arm has been observed $\ell=N_k(\tau)\geq1$ times.
Let $q=s_j-1\leq\ell$ be the largest preceding candidate and write $d=\ell-q$.
After $m$ post-change observations, the candidate's two segments have
lengths $n_0=q$ and $n_1=d+m$, with $n=q+d+m$.
Therefore,
\[
D_{j,n}
=
\frac{m}{d+m}|\theta_1-\theta_0|.
\]
The $d$ pre-change observations in the second segment dilute the visible
change by the factor $m/(d+m)$.
The actual change occurs between arm-local observations $\ell$ and
$\ell+1$.
Since $s_j=q+1$ is the last grid point at or before this
change, the next grid point satisfies $s_{j+1}>\ell+1$.
Using the grid recurrence $s_{j+1}=\lceil1.25s_j\rceil$,
\[
s_{j+1}
=
\lceil1.25(q+1)\rceil
>
\ell+1
\]
implies
\[
1.25(q+1)>\ell+1,
\qquad\text{and hence}\qquad
5q>4\ell-1.
\]
Since $q$ and $\ell$ are integers, $5q\geq4\ell$. Therefore,
\[
q\geq\frac{4\ell}{5},
\qquad
d=\ell-q\leq\left\lfloor\frac q4\right\rfloor.
\]
So, at most $q/4$ pre-change observations fall in the candidate's
second segment.
Because Theorem~\ref{thm:adaptive-power} holds simultaneously over all
completed observation counts, the random value of $\ell$ created by adaptive
sampling does not require conditioning.

The resulting certificate is rather conservative.
For $K=4$, $\alphaE=\xi=0.05$, and a change of size $0.5$, the
preceding candidate first satisfies \eqref{eq:adaptive-power} after
$m=479$ post-change observations when $\ell=1000$ and after $m=755$
when $\ell=300$.
When $\ell=100$, it never holds because its left side is
at most $13.3864$, below the right side's minimum of $16.7283$.
These are thresholds for a sufficient condition, but not observed detection delays.
Note that authorization may occur even when the condition fails, and the policy may
not collect enough observations for the condition to apply.

\paragraph{When can we use a sharper condition?}

When the candidate and both segment lengths are fixed before observing
rewards, and the change occurs at that candidate, we need to control sampling
error for only two fixed segments.
This therefore yields a sharper sufficient condition.

\begin{theorem}[Finite-sample authorization after an aligned change]
\label{thm:finite-power}
Fix candidate $j$ and an integer $m\geq1$ before observing the rewards,
and let $\ell=s_j-1$.
Suppose $X_{k,1},\ldots,X_{k,\ell+m}$ are independent, with the first
$\ell$ rewards distributed as $\mathrm{Bernoulli}(\theta_0)$ and
the next $m$ as $\mathrm{Bernoulli}(\theta_1)$.
For $\xi\in(0,1)$, define
\[
 a_\ell
 =
 \sqrt{\frac{\log(4/\xi)}{2\ell}},
 \qquad
 a_m
 =
 \sqrt{\frac{\log(4/\xi)}{2m}},
 \qquad
 \Delta
 =
 |\theta_1-\theta_0|.
\]
If
\begin{equation}
 2\frac{\ell m}{\ell+m}
 (\Delta-a_\ell-a_m)_+^2
 \ge
 \log(K/\alphaE)
 +\log(j(j+1))
 +\log(\ell+1)
 +\log(m+1),
 \label{eq:finite-power}
\end{equation}
where $(x)_+=\max\{x,0\}$, then
\[
 \Pr_{\theta_0\to\theta_1}
 \left\{
 E_{k,\ell+m}\ge K/\alphaE
 \right\}
 \ge
 1-\xi.
\]
Under the same condition, for any policy that chooses which arm
to pull using only past observations and its own randomization,
and each $t\leq T$,
\[
\Pr_{\theta_0\to\theta_1}
\left\{
N_k(t)\geq\ell+m,\ C^E_{k,t}=0
\right\}
\leq\xi.
\]
\end{theorem}

\paragraph{Why does the condition work?}
The sufficient condition in Theorem~\ref{thm:finite-power}, stated in
Equation~\eqref{eq:finite-power}, works through three steps.
It first shows that the change remains visible despite sampling noise, then converts the
observed difference into evidence, and finally shows that this evidence is
large enough to cross the authorization threshold.
First, let $\widehat p_0$ and $\widehat p_1$ be the
empirical success rates before and after the change.
Hoeffding's inequality bounds how far each empirical rate can be from its true probability.
Applying it to both segments and taking a union bound gives, with probability at least $1-\xi$,
\[
|\widehat p_0-\theta_0|\leq a_\ell,
\qquad
|\widehat p_1-\theta_1|\leq a_m.
\]
Since the true probabilities differ by $\Delta$, these two bounds imply
\[
|\widehat p_1-\widehat p_0|
\geq
(\Delta-a_\ell-a_m)_+.
\]

Second, this observed difference must produce enough evidence against a common
reward probability.
Pinsker's inequality converts the difference between the two empirical rates into a lower bound
\[
I_{j,\ell+m}
\geq
2\frac{\ell m}{\ell+m}
|\widehat p_1-\widehat p_0|^2.
\]

As the final step, the evidence must remain above the authorization threshold after the
candidate and averaging costs are subtracted.
Equations \eqref{eq:finite-power} and \eqref{eq:information-budget} give
\[
\log E_{k,\ell+m}\geq\log(K/\alphaE),
\]
so $E_{k,\ell+m}\geq K/\alphaE$.
Since the preceding sampling bounds hold with probability at least
$1-\xi$, we get
\[
\Pr_{\theta_0\to\theta_1}
\left\{
E_{k,\ell+m}\geq K/\alphaE
\right\}
\geq1-\xi.
\]

To obtain the allocation rule, note that, pathwise,
\[
\left\{
N_k(t)\geq\ell+m,\ C^E_{k,t}=0
\right\}
\subseteq
\left\{
E_{k,\ell+m}<K/\alphaE
\right\}.
\]
Therefore, if the arm has completed observation $\ell+m$ and its
evidence crossed at that update, the permanent latch cannot still
be at zero.
The event on the right depends only on the fixed first $\ell+m$
entries of the arm's reward sequence.
Its probability is at most $\xi$, regardless of when the allocation
rule reveals those entries.
To clarify, this is not a bound conditional on the observed pull count.
This proves Theorem~\ref{thm:finite-power}. \qed

\paragraph{Does the authorization latch affect the controller?}
To isolate the latch's role, compare \method{} with a version that sets
$C^E_{k,t}=1$ for every arm and round while leaving the rest of the controller
unchanged.
Recall that $r_{k,n}$ measures recent disagreement with the
long-history model, while $w_{k,t}$ determines how strongly the short-memory
posterior influences the arm's score.

\begin{proposition}[Authorization changes the controller]
\label{prop:load-bearing}
Let $A_1$ be the arm selected in round one. For $T\geq2$, its first reward
produces $r_{A_1,1}>0$, whether that reward is zero or one. The comparison
without authorization therefore has $w_{A_1,2}>0$. In \method{}, no arm can
yet be authorized, so $w_{k,2}=0$ for every arm and the round-two scoring rule
is exactly that of \anchor{}.

More generally, consider an authorized arm $k$ in round $t$ and let
$n=N_k(t-1)$.
Holding its normalization scale $s^M_{k,n}$ and retained-data
factor $R_{k,t}$ fixed, the decision weight $w_{k,t}$ increases strictly with
the magnitude of its averaged prediction error, $|M_{k,n}|$, whenever
\[
\epsM<|M_{k,n}|<s^M_{k,n}.
\]
\end{proposition}

\paragraph{Why does the difference appear after one reward?}
Let $d=1-\lambda_D$ be the weight placed on the newest prediction error.
The cap $\lambda_D\leq0.999$ ensures that $d\geq10^{-3}$.
Before its first reward, the selected arm $A_1$ has a long-history posterior mean of $1/2$.
Either possible reward therefore produces an error of magnitude $1/2$, giving
\[
|M_{A_1,1}|=\frac d2,
\qquad
Q_{A_1,1}=d^2,
\qquad
r_{A_1,1}
=
\frac{1}{
\sqrt{2\log(K\pi^2/(3\kappaM))}
}
>0.
\]
As defined in Appendix~\ref{app:algorithm}, $\kappaM=0.05$ is the fixed normalization constant used in
$s^M_{k,n}$.
It determines how the averaged prediction error is converted into relevance.

The first candidate change-point is $s_1=2$, so no candidate has been reached
after only one pull of the arm.
Its evidence therefore remains $E_{A_1,1}=1<K/\alphaE$, and the authorization latch remains zero.
In the comparison without authorization, the latch is instead fixed at one.
Since $r_{A_1,1}>0$, the round-two gate satisfies
$g_{A_1,2}=r_{A_1,1}>0$, which in turn gives $w_{A_1,2}>0$.
In \method{}, the gate and decision weight remain zero.

We next prove the monotonicity claim.
Recall that $g=C^E r$ is the arm's
gate.
As defined in Appendix~\ref{app:algorithm}, the quantity $g^\nu$ is a transformation of the same gate used in the decision weight.
Since $\nu=0.3<1$, it magnifies values between zero and one while preserving the endpoints.

For fixed $R\in[0,1]$, let $c=\zeta R$ and write the decision weight as
\[
\omega(g)=g^\nu+cg(1-g^\nu).
\]
Since $0\leq c\leq\zeta=0.8<1$, for $g\in(0,1]$,
\[
\omega'(g)
=
\nu g^{\nu-1}(1-cg)
+c(1-g^\nu)
>0.
\]
For an authorized arm, $C^E=1$ and hence $g=r$.
Before relevance reaches its cap, $r=|M|/s^M$.
So, when $s^M$ and $R$ are held fixed, the decision weight
increases strictly with $|M|$.
This proves that authorization changes the controller.
It need not change the selected arm on every trajectory or improve regret universally. \qed

Together, these results characterize how much evidence a candidate can
provide, when that evidence is sufficient for authorization, and how the latch
changes the controller.
The authorization conditions presented only concern information from the observations
collected from an arm.
Theorem~\ref{thm:adaptive-power} allows their timing to depend on
earlier rewards, while Theorem~\ref{thm:finite-power} gives a sharper
bound under its fixed-prefix, aligned-change assumptions.
Neither result guarantees that a changed arm will be sampled enough
to obtain authorization by a specified calendar round.

\section{Experimental design and additional results}
\label{app:experiments}

Section~\ref{sec:experiments} reports the main empirical findings.
This appendix explains how the environments change across geometries, gives the
complete gate comparison, and reports the comparison with baselines from the
literature.

We use two suites because they answer different questions.
The registered suite was designed to measure the contribution of authorization across a
broad range of arm counts and horizons.
The registered suite follows a protocol fixed before the primary evaluation.
The replay suite was assembled from non-stationary bandit families used in our earlier baseline study and is used
to compare \method{} with established methods.
We later repeated the gate comparison on the replay suite to test whether the registered result
transferred to that different population.

\begin{table}[H]
\centering
\small
\caption{Roles of the two experimental suites. The suites use different
environment generators, so regret levels should not be compared across rows.}
\label{tab:suite-roles}
\begin{tabular}{@{}>{\raggedright\arraybackslash}p{0.14\textwidth}
>{\raggedright\arraybackslash}p{0.31\textwidth}
>{\raggedright\arraybackslash}p{0.25\textwidth}
>{\raggedright\arraybackslash}p{0.19\textwidth}@{}}
\toprule
Suite & Primary purpose & Geometries & Evaluation groups \\
\midrule
Registered &
Contribution of authorization across scale &
$K\in\{4,8,16,32,64\}$ and
$B=T/K\in\{250,625,1250,2500\}$ &
16,380 non-stationary and 1,820 stationary \\
Replay &
Comparison with baselines from the literature and a locked gate transfer &
$(K,T)=(4,2500)$ and $(32,20000)$, both with $B=625$ &
4,500 non-stationary \\
\bottomrule
\end{tabular}
\end{table}

For arm $k$ at round $t$, let $\mu_{k,t}$ be its Bernoulli reward probability
and let $A_t$ be the selected arm.
We measure normalized dynamic pseudo-regret,
\[
 \bar R_T
 =
 \frac1T\sum_{t=1}^T
 \left(
 \max_k\mu_{k,t}-\mu_{A_t,t}
 \right).
\]
A mean trajectory is the complete table of Bernoulli reward probabilities
$\{\mu_{k,t}\}$ for every arm and round.
A potential-reward table contains the corresponding binary reward that each arm would produce if selected in
each round.
A policy observes only the entry for the arm it actually selects.
A geometry specifies the number of arms $K$ and horizon $T$, with
$B=T/K$ denoting the nominal rounds available per arm.
The initial policy seed initializes the random-number stream used for posterior sampling.

An experimental group fixes one mean trajectory, one potential-reward table,
one geometry, and one initial policy seed.
The compared policies share these inputs, so their regret difference forms a paired contrast.
For gate comparisons, we report
\[
 \Delta(c)
 =
 \bar R_T(c)-\bar R_T(\text{\method{}}).
\]
Positive values favor \method{}, while negative values favor comparator $c$.
Lower regret is always better.

\subsection{What does authorization contribute?}

\paragraph{How is the registered suite constructed?}
The suite crosses nine non-stationary mechanisms with 20 geometries.
Stationarity forms a tenth mechanism used for the stationary comparison.
Table~\ref{tab:registered-environments} summarizes the nine non-stationary
mechanisms, while Figure~\ref{fig:primary-environment-atlas} shows one frozen
four-arm example from each.

\begin{table}[p]
\centering
\footnotesize
\caption{Non-stationary mechanisms in the registered suite.}
\label{tab:registered-environments}
\begin{tabular}{@{}p{0.27\textwidth}p{0.65\textwidth}@{}}
\toprule
Mechanism & Mean-trajectory construction \\
\midrule
Isolated incumbent decrease &
The leading arm drops once past an unchanged comparison arm \\
Isolated challenger increase &
A trailing arm rises once past an unchanged comparison arm \\
Sparse sustained shock &
A small subset of arms drops once and remains lower for the rest of the horizon \\
Random breakpoints &
Two to five breakpoints change one arm at a time, alternating direction \\
Global switching &
Across two to six equal segments, the arm means are cyclically reassigned \\
Rapid dense switching &
Every arm changes at each of many closely spaced breakpoints \\
Abrupt change and drift &
All arms drift smoothly, and a subset also undergoes one abrupt decrease \\
Asynchronous per-arm switching &
Arms change once at different times, with shifts alternating in direction \\
Gradual drift &
All arm means vary smoothly along sinusoidal paths \\
\bottomrule
\end{tabular}
\end{table}

\begin{figure}[p]
  \centering
  \includegraphics[width=\textwidth]{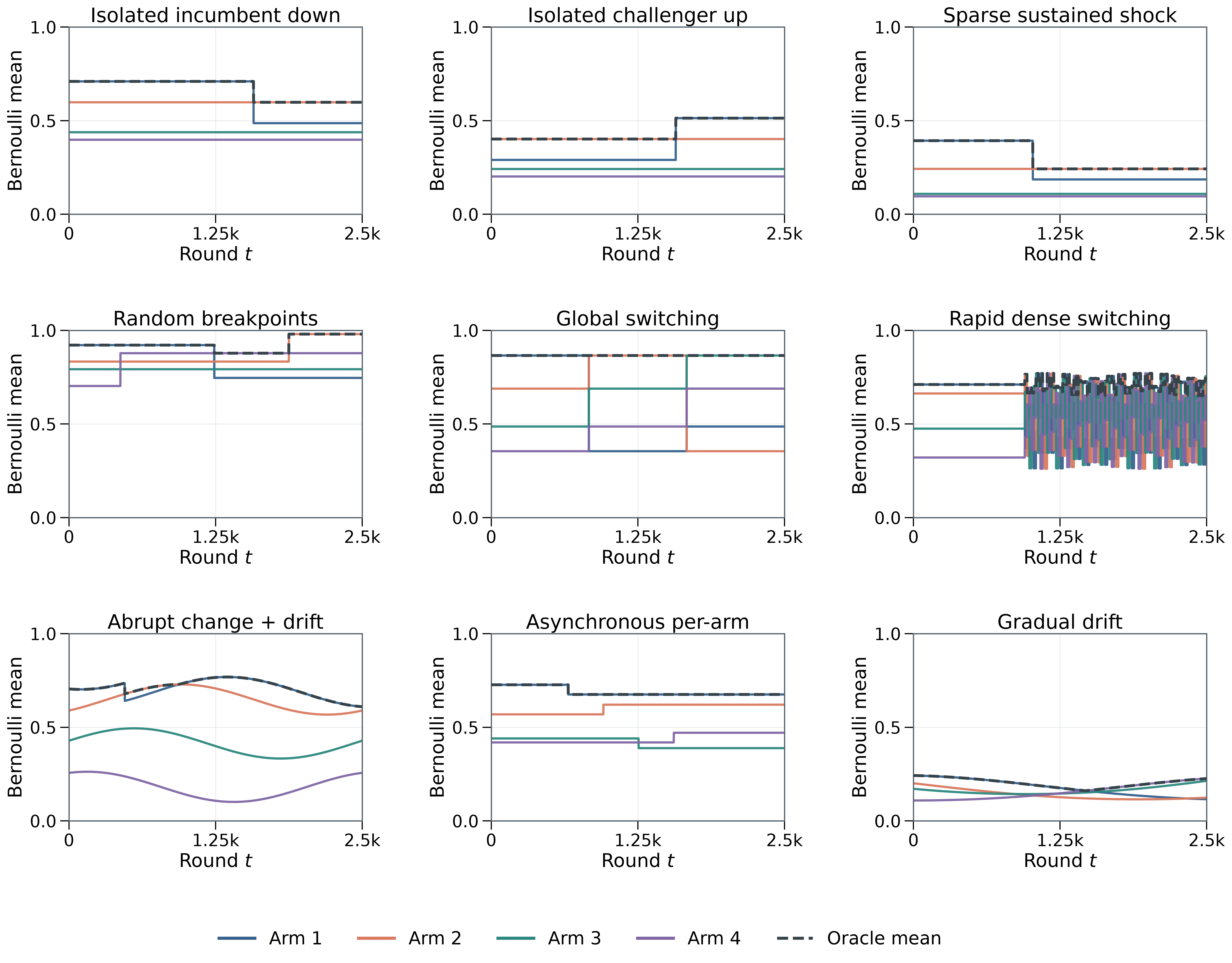}
  \caption{One frozen registered-suite environment for each non-stationary
  mechanism at $K=4$, $B=625$, and $T=2{,}500$. Blue, coral, green, and purple
  identify arms 1--4 across panels using the palette of
  Figure~\ref{fig:authorization-rank-summary}. The dashed black line is
  the instantaneous oracle mean. Within this atlas, color identifies arm
  number only and does not encode an arm's role.}
  \label{fig:primary-environment-atlas}
\end{figure}

Each mechanism and geometry contains 91 randomly generated environments and
27 fixed stress environments.
The stress environments use prespecified boundary settings of the generator parameters rather than randomly sampled
settings.
The primary regret analysis uses the random environments, giving $9\times20\times91=16{,}380$ non-stationary groups and
$1\times20\times91=1{,}820$ stationary groups.

The fixed environments form a separate census of 5,400 groups and are not used for the primary mean or confidence interval.
The pooled contrast across their 4,860 non-stationary groups was instead used as a predeclared breadth
check.
Development, validation, and evaluation used disjoint environments and potential-reward tables.

\paragraph{How does an environment change with $K$ and $T$?}
The $K=4$ environments illustrate each mechanism but are not copied into the
larger geometries.
At every $K$, the generator creates a new $K$-arm mean
trajectory from a deterministic evaluation seed.
The horizon is $T=KB$, so increasing $K$ at fixed $B$ increases the total horizon while retaining the
same nominal number of rounds per arm.

For mechanisms with discrete changes, the round at which each reward
probability changes is specified as a fraction of $T$.
At larger geometries, the generator creates new
trajectories under the same mechanism rules.
Discrete changes are located relative to the horizon through $t/T$, and mechanisms designed to change one
arm continue to affect only one arm.

\paragraph{Does authorization help on the registered suite?}
Removing only the authorization latch gives the open-ramp policy,
$g_{k,t}=r_{k,N_k(t-1)}$.
The short-memory posterior is then eligible to influence decisions as soon as its relevance becomes positive.

We chose this comparison to isolate the contribution of authorization.
The two policies are otherwise identical, but the open ramp allows the relevance
score to act without first requiring an e-process crossing.
We fixed this comparison and its margin before examining the evaluation results.
The margin was cleared only if the lower one-sided $95\%$ confidence bound for
$\Delta(\text{open ramp})$ exceeded $0.0002$.
We formed this bound by resampling environments within each mechanism and geometry and recomputing the
paired regret difference.

Across the 16,380 random non-stationary groups, mean regret was $0.0390$ for
\method{} and $0.0540$ for the open ramp.
The paired contrast was
\[
 \Delta(\text{open ramp})
 =
 0.01497,
 \qquad
 95\%\ \mathrm{CI}
 =
 [0.01461,0.01533],
\]
which clears the predeclared margin.
All nine mechanism means favored \method{}, with contrasts from $0.00164$ to $0.02460$.
All four rounds-per-arm means also favored \method{}, with contrasts from $0.01282$ to
$0.01658$.
The fixed stress census had the same pooled ordering.

The result changes with arm count.
The left panel of Figure~\ref{fig:authorization-rank-summary} shows that the open ramp was lower
at $K=4$, while \method{} was lower at every $K\ge8$.
The registered contrast increased from $-0.00289$ at $K=4$ to $0.03860$ at $K=64$.
Authorization therefore helped on average and in four of the five arm-count rows, but not in
every registered geometry.

\paragraph{What happens when the arms are stationary?}
Only the registered suite contains stationary regret experiments.
In these environments, every arm's reward probability remains fixed throughout the
horizon, although its binary rewards still vary.
Across the 1,820 random stationary groups, mean regret was $0.0115697$ for \method{} and $0.0115673$
for \anchor{}.
Both round to $0.01157$.
During these experiments, \method{} authorized exactly one arm.
The open ramp had mean regret $0.02937$, about two and a half times that of either policy.

Separate from both experimental suites, we also examined how often any arm crossed
the authorization threshold under the prior-predictive null of
Theorem~\ref{thm:main}.
We used three $(K,T)$ settings:
$(4,2500)$, $(32,20000)$, and $(8,50000)$, each with 4,000 runs.
In every run, each arm's reward probability was drawn independently from
$\mathrm{Beta}(1,1)$ and then fixed.
At least one arm crossed the authorization threshold in $4$ of the $4,000$
runs at $(K,T)=(4,2500)$, $1$ of the $4,000$ runs at $(32,20000)$, and $4$ of
the $4,000$ runs at $(8,50000)$.

Since the arm probabilities were redrawn for every run, these rates average
over the $\mathrm{Beta}(1,1)$ prior.

\paragraph{How often does the permanent latch activate?}
By the end of the registered non-stationary runs, the mean fraction of
authorized arms was $9.97\%$. The fraction was $5.77\%$ in gradual drift and
$11.93\%$ in mixed abrupt change and drift. On the replay suite, the overall
fraction was $38.27\%$, including $15.38\%$ in random drift and $14.42\%$ in
the mixed family.

Each fraction counts every arm that was authorized at least once before the
run ended.
Since authorization is never withdrawn, an arm remains counted
even if its later rewards again agree with the stationary model.
The fractions therefore describe whether authorization ever occurred, not whether the
evidence remained current at the time when the run ended.

\paragraph{How do alternative authorization rules compare?}
\label{app:glr-comparator}
We next ask whether the policy's behavior depends on how authorization is
granted.
We replace the e-process crossing with two generalized likelihood-ratio (GLR) rules and leave the rest of the policy unchanged.
All three policies use the same candidate change-points, relevance score,
decision weight, posterior updates, potential-reward table, and initial policy
seed.
Thus, only the rule used to turn on the authorization latch changes.

For candidate $j$, the Bernoulli GLR statistic is
\[
 G_{j,n}
 =
 n_0\kl(\widehat p_0\Vert\widehat p)
 +
 n_1\kl(\widehat p_1\Vert\widehat p).
\]
Here, $n_0=s_j-1$ and $n_1=n-n_0$ are the numbers of arm-local pulls
before and after the candidate change-point.
The corresponding success rates are $\widehat p_0$ and $\widehat p_1$,
while $\widehat p$ is the success rate over all $n$ pulls.
The Bernoulli Kullback--Leibler divergence $\kl$ is defined in
Appendix~\ref{app:evidence}.

We consider two GLR authorization rules.
First, the certified GLR authorizes when the maximum over the geometric grid
crosses
\[
 b_{\mathrm{cert}}(n,\delta)
 =
 2\mathcal T\!\left(
 \frac{\log(3n^{3/2}/\delta)}{2}
 \right)
 +6\log(1+\log n),
 \qquad
 \delta=\alphaE/K.
\]
Here, $\delta$ assigns an equal share of $\alphaE$ to each arm.
The function $\mathcal T$ is defined in Equation~(3) of
\citet{besson2022glr}.
The complete expression for $b_{\mathrm{cert}}$ is their threshold
$\beta(n,\delta)$ from Lemma~2, Equation~(4).
The certified boundary was proved for a detector that checks every possible
split after each pull.
Our implementation checks only the smaller geometric grid, so using the same boundary remains valid and is potentially more
conservative.

The guarantee also holds for any fixed stationary Bernoulli probability.
That is, the probability of crossing is controlled separately for every
possible stationary arm mean.
Theorem~\ref{thm:main}, by contrast, averages over arm means drawn from the $\mathrm{Beta}(1,1)$ prior.
The certified GLR therefore covers a broader stationary model than the e-process guarantee in
this paper.

Second, the practical GLR uses
\[
 b_{\mathrm{prac}}(n,\delta)
 =
 \log(n^{3/2}/\delta).
\]
It is also specified without any calibration, although it has no claimed
finite-sample false-authorization guarantee.
To clarify, neither GLR threshold was fitted to these evaluation outcomes.

\begin{table}[H]
\centering
\small
\caption{Gate comparison on both non-stationary suites. Each difference is
the comparator's regret minus \method{}'s regret, so positive values favor \method{}. Lower regret is better.}
\label{tab:gate-both-suites}
\begin{tabular}{@{}lrrrr@{}}
\toprule
& \multicolumn{2}{c}{Registered}
& \multicolumn{2}{c}{Replay} \\
\cmidrule(lr){2-3}\cmidrule(l){4-5}
Policy & Regret & Difference & Regret & Difference \\
\midrule
\anchor{}       & 0.0494 & $+0.01046$ & 0.2157 & $+0.06277$ \\
Open ramp       & 0.0540 & $+0.01497$ & 0.1414 & $-0.01154$ \\
\method{}       & 0.0390 & reference  & 0.1530 & reference \\
Certified GLR   & 0.0418 & $+0.00282$ & 0.1750 & $+0.02209$ \\
Practical GLR   & 0.0392 & $+0.00017$ & 0.1529 & $-0.00006$ \\
\bottomrule
\end{tabular}
\end{table}

The certified GLR was more conservative and had higher regret than \method{}
on both suites.
The practical GLR was close to \method{} on both.
On the registered suite, its regret was higher by $0.00017$, with interval
$[0.00009,0.00026]$.
On the replay suite, its regret was lower by $0.00006$,
with interval $[-0.00052,0.00040]$.
Relative to \method{}, the practical GLR had $0.44\%$ higher regret on the
registered suite and $0.04\%$ lower regret on the replay suite.
These small pooled differences do not imply cellwise equivalence.
The practical GLR was lower in 13 of the 18 replay cells.

An earlier sensitivity allowed the GLR to use a separately calibrated
threshold in each of the 20 geometries.
For each geometry, we chose the threshold so that the GLR crossed as often as
\method{} on the stationary calibration runs.
Across 5,460 calibration runs, \method{} never crossed in 17 geometries.
In those geometries, matching zero crossings only required placing the GLR
threshold above the largest statistic observed during calibration.
On 1,820 held-out stationary runs, the fitted GLR authorized seven times and
\method{} once.
Across the non-stationary suite, its regret was $0.0387$, compared with $0.0390$
for \method{}, a reduction of $0.66\%$.
The fitted GLR's regret minus \method{}'s was $-0.00026$, with a 95\% interval
$[-0.00034,-0.00017]$.
This count-matching sensitivity gave the GLR geometry-specific Monte Carlo
calibration but no analytic anytime-valid guarantee.

\paragraph{Does the registered result transfer to the replay suite?}
\label{app:replay-gate}
No. On the replay suite, the open ramp was lower in 15 of 18 cells, eight of
nine families, and 3,204 of 4,500 groups.
Its mean regret advantage was $0.01154$, with interval $[0.01089,0.01218]$.
The two negative replay suite points in the left panel of Figure~\ref{fig:authorization-rank-summary}, in orange, show that the
open ramp was lower at both evaluated arm counts.

We added the replay gate comparison to ask whether the registered result carried
to the environments used for the baseline comparison.
It did not.
On the registered suite, \method{} beat the open ramp, while the open ramp beat
\method{} on the replay suite.
The practical GLR remained close to \method{} on both suites, whereas the
certified GLR had higher regret.
Authorization therefore matters, but its effect on regret depends on the
environments being studied.
The comparison only estimates authorization’s overall effect on regret.
It cannot determine whether the difference comes from delayed adaptation,
short-memory forgetting, or subsequent action choices.
Note that these experiments also do not establish any authorization rule as universally best.

\subsection{How does \method{} compare with established baselines?}
\label{app:historical-environments}
\label{app:literature-replay}

\paragraph{How is the replay suite constructed?}
The replay suite evaluates nine non-stationary families at two geometries,
$(K,T)=(4,2500)$ and $(32,20000)$.
Both have $B=T/K=625$, the same nominal number of rounds per arm, giving a total of 18 family--geometry combinations.
Within each combination, every policy is evaluated on the same 250 experimental
settings, yielding a total of 4,500 matched comparisons.

Periods and segment lengths are multiplied by $K/4$ when moving from $K=4$
to $K=32$.
Switching hazards are multiplied by $4/K$, and diffusion increments are divided by $\sqrt{K/4}$.
These changes keep the amount of non-stationarity relative to the rounds available per arm approximately
comparable.
The larger geometry is a newly generated environment, not an eightfold repetition of the four-arm trajectory.

\begin{table}[p]
\centering
\footnotesize
\caption{Families in the replay suite.}
\label{tab:historical-environments}
\begin{tabular}{@{}p{0.22\textwidth}p{0.70\textwidth}@{}}
\toprule
Family & Mean-trajectory construction \\
\midrule
Slow sine &
Smooth arm-specific sinusoidal means with periods 1,000 and 8,000 \\
Fast sine &
The same construction with periods 100 and 800 \\
Abrupt cycle &
Staggered arm-specific level changes in cycles of length 250 and 2,000 \\
Global switching &
All arm means change together with per-round hazard $0.08/K$ \\
Per-arm switching &
Each arm changes independently with per-round hazard $0.08/K$ \\
Local shocks &
Oracle-changing shocks affect a subset of arms \\
Random breakpoints &
Each arm has 4--16 breakpoints and jumps of at least $0.10$ \\
Random drift &
Independent stochastic drift with time-varying sinusoidal amplitude \\
Mixed regime &
Groups of arms are stable, drifting, transiently bursty, or switching \\
\bottomrule
\end{tabular}
\end{table}

Table~\ref{tab:historical-environments} summarizes how the replay families are
constructed, while Figure~\ref{fig:historical-environment-atlas} shows one
frozen four-arm example from each family.

\begin{figure}[p]
  \centering
  \includegraphics[width=\textwidth]{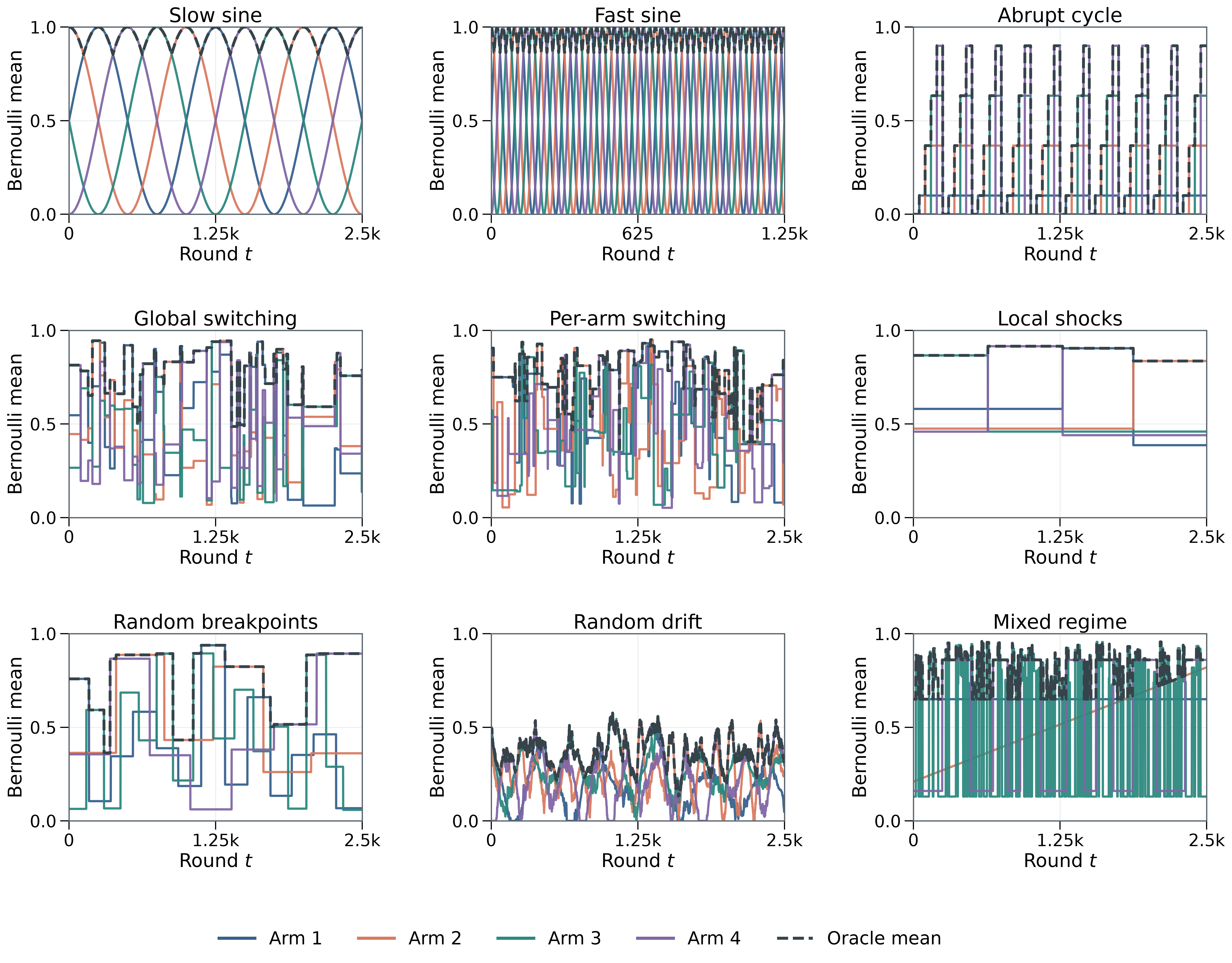}
  \caption{One frozen environment from each replay family at $K=4$ and
  $T=2{,}500$. Blue, coral, teal, and plum identify arms 1--4 across panels
  using the palette of Figure~\ref{fig:authorization-rank-summary}. The
  dashed charcoal line is the instantaneous oracle mean. The examples use
  the frozen visualization seed 1,840,201 and illustrate the family
  constructions, not their empirical frequency. The fast-sine panel shows
  rounds 0--1,250 so that its shorter period remains visible. Every other
  panel shows the full 2,500-round horizon.}
  \label{fig:historical-environment-atlas}
\end{figure}

Each slow-sine, fast-sine, and abrupt-cycle geometry holds one mean trajectory
fixed across 250 reward simulations.
For each of the other six families and their geometries, we generate 50 different
mean trajectories.
We then simulate rewards five times from each trajectory, again giving 250
policy comparisons for that family and geometry.
This separates differences across mean trajectories from randomness in the
rewards generated from the same trajectory.

The original replay analysis calculated each interval from 50,000 repeated
draws.
Within every combination of environment family and geometry, each draw
selected 50 groups of five policy comparisons with replacement.
In the twelve settings with random trajectories, each group contained
one mean trajectory and its five reward repetitions.
In the six settings with a fixed trajectory, each group contained five
repetitions of that trajectory.

After examining the results, we recalculated the primary regret difference
defined as \method{} minus the open ramp.
Each new draw selected 250 individual
policy comparisons with replacement in the fixed-trajectory settings.
In the random-trajectory settings, each new draw selected 50 mean trajectories with replacement
and retained all five reward repetitions for each trajectory.
The corrected interval was $[0.010886,0.012184]$, compared with the original interval
$[0.010899,0.012181]$.
The estimated difference did not change.
The main text reports the corrected primary interval.
The GLR intervals use the original calculation.
The correction therefore does not change the replay-suite finding that the
open ramp has lower regret than \method{}.

\paragraph{Which methods are included?}
The comparison contains \method{} and local implementations of 16 baselines
from the literature. These are TS and \anchor{} \citep{thompson1933,may2012optimistic},
dTS and dOTS \citep{raj2017taming}, Dynamic TS \citep{gupta2011dynamic}, AFF-TS
\citep{lu2017adaptive}, REXP3 \citep{besbes2014rexp3}, SW-UCB and D-UCB
\citep{garivier2008ucb}, M-UCB \citep{cao2018mucb}, CUSUM-UCB
\citep{liu2017cusum}, GLR-klUCB \citep{besson2022glr}, Global-CTS
\citep{mellor2013globalcts}, a one-hot implementation of D-LinUCB
\citep{russac2019weighted}, and Beta-SWTS and $\gamma$-SWGTS
\citep{fiandri2024swts}.

We implemented each baseline ourselves, following its source algorithm and
recommended settings as faithfully as possible.
These are careful implementations rather than runs of the original
authors' code.
No method was tuned on the replay results, and \method{} used one fixed
configuration across all 18 family--geometry combinations.
Since the methods were developed at different times, we do not claim that
each received the same amount of tuning.

Within each of the 18 family--geometry combinations, the displayed methods are
ranked by mean regret, with lower regret receiving a better rank.
We report each method's mean regret and mean rank separately at the two
geometries and across all 18 combinations.
Each family receives equal weight within a geometry, and every combination
receives equal weight in the overall mean.

\begin{table}[H]
\centering
\footnotesize
\setlength{\tabcolsep}{3.5pt}
\caption{Replay performance at each geometry and across all 18
family--geometry combinations. Regret and rank are means, lower values are
better, and ranks are recomputed among the 17 displayed methods.}
\label{tab:literature-ranks}
\begin{tabular}{@{}lrrrrrr@{}}
\toprule
& \multicolumn{2}{c}{$K=4$, $T=2{,}500$}
& \multicolumn{2}{c}{$K=32$, $T=20{,}000$}
& \multicolumn{2}{c}{Overall} \\
\cmidrule(lr){2-3}\cmidrule(lr){4-5}\cmidrule(l){6-7}
Method & Regret & Rank & Regret & Rank & Regret & Rank \\
\midrule
GLR-klUCB        & 0.1070 & 3.56  & 0.1027 & 2.56  & 0.1048 & 3.06 \\
SW-UCB           & 0.0991 & 4.56  & 0.1006 & 3.00  & 0.0999 & 3.78 \\
M-UCB            & 0.0969 & 4.11  & 0.1140 & 5.22  & 0.1055 & 4.67 \\
D-UCB            & 0.0965 & 3.89  & 0.1259 & 7.56  & 0.1112 & 5.72 \\
Global-CTS       & 0.1142 & 5.78  & 0.1194 & 5.89  & 0.1168 & 5.83 \\
dOTS             & 0.1203 & 6.33  & 0.1453 & 8.56  & 0.1328 & 7.44 \\
Dynamic TS       & 0.1377 & 8.78  & 0.1280 & 7.11  & 0.1329 & 7.94 \\
CUSUM-UCB        & 0.1496 & 9.89  & 0.1218 & 6.11  & 0.1357 & 8.00 \\
Beta-SWTS        & 0.1303 & 9.44  & 0.1349 & 7.44  & 0.1326 & 8.44 \\
D-LinUCB one-hot & 0.1003 & 5.22  & 0.1964 & 12.00 & 0.1484 & 8.61 \\
dTS              & 0.1344 & 9.22  & 0.1521 & 10.11 & 0.1432 & 9.67 \\
\method{}        & 0.1782 & 12.56 & 0.1277 & 6.89  & 0.1530 & 9.72 \\
AFF-TS           & 0.1524 & 11.00 & 0.1492 & 9.89  & 0.1508 & 10.44 \\
$\gamma$-SWGTS   & 0.1835 & 13.33 & 0.2445 & 15.00 & 0.2140 & 14.17 \\
\anchor{}        & 0.2136 & 14.44 & 0.2178 & 14.00 & 0.2157 & 14.22 \\
TS               & 0.2204 & 15.33 & 0.2235 & 14.78 & 0.2220 & 15.06 \\
REXP3            & 0.2250 & 15.56 & 0.3117 & 16.89 & 0.2684 & 16.22 \\
\bottomrule
\end{tabular}
\end{table}

Across all 18 combinations, \method{} had mean regret $0.1530$ and mean rank
9.72 of 17.
At $K=4$ its mean regret was $0.1782$ and its mean rank was 12.56, while at
$K=32$ they were $0.1277$ and 6.89.
It improved on \anchor{} in 17 of 18 combinations, reducing overall mean regret
from $0.2157$ to $0.1530$.
This is a meaningful improvement over the protected anchor, but the overall
ranking remains in the middle of the comparison rather than at the top.

\begin{figure}[H]
  \centering
  \includegraphics[width=\textwidth]{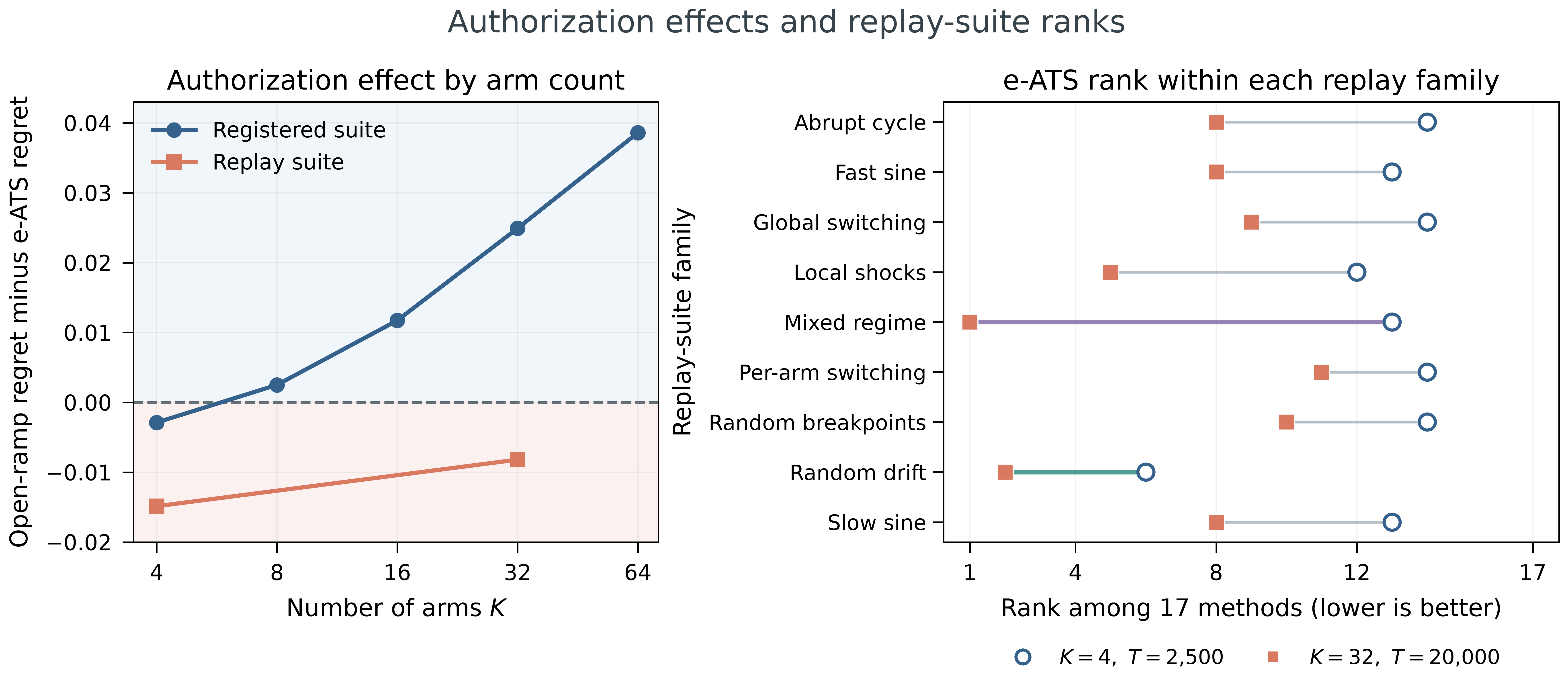}
  \caption{Authorization effects and replay-suite ranks. Left: the open ramp's
  regret minus \method{}'s regret by arm count. Positive values favor
  \method{}, while negative values favor the open ramp. Lines connect
  descriptive arm-count means within a suite. Right: \method{}'s ranks among
  the 17 displayed methods in each replay family. Circles show
  $(K,T)=(4,2500)$ and squares show $(K,T)=(32,20000)$. The colored
  connectors highlight mixed regime and random drift, where \method{} ranked
  first and second at the larger geometry. Lower rank is better.}
  \label{fig:authorization-rank-summary}
\end{figure}

Table~\ref{tab:literature-ranks} shows that the
mean rank was 12.56 at $(K,T)=(4,2500)$ and 6.89 at
$(K,T)=(32,20000)$.
At the larger geometry, \method{} ranked better in all nine families, including
first in the mixed family and second in random drift.
These are the strongest combinations for \method{} and suggest that selective
adaptation may be useful when different arms evolve in different ways.

In our experiments, the geometries differ in arm count, total horizon, and generator scaling.
The rank change cannot be attributed to any one of those differences.
The gate transfer also shows that the open ramp outperformed \method{} on this same replay population.
The first- and second-place cells do not establish that the
e-process gate caused those rankings.

Overall, authorization changes regret, but whether it helps depends on the environments
being studied.
On the registered suite, authorization reduced regret relative to opening the short-memory channel immediately, except at $K=4$.
On the replay suite, the open ramp performed better, while \method{} still improved
substantially over \anchor{} and performed especially well in two $K=32$
cells.
The practical GLR closely matched \method{}.
The certified GLR had higher regret and provided the stronger fixed-mean
guarantee, while the fitted GLR slightly outperformed \method{} after receiving
geometry-specific calibration.
In these suites, \method{} differed little in regret from the untuned practical
detector while retaining its prior-predictive guarantee.

\end{document}